\RequirePackage{fix-cm}
\documentclass[smallextended]{svjour3}       
\smartqed  
\usepackage{graphicx}
\usepackage{amsmath,amssymb}
\numberwithin{equation}{section}
\usepackage{subfigure}
\usepackage{multirow}
\usepackage{booktabs}
\usepackage{xcolor}

\begin{document}

\title{Analysis of Estimating the Bayes Rule for Gaussian Mixture Models with a Specified Missing-Data Mechanism
}

\titlerunning{Estimating the Bayes Rule under a Specified Missing-Data Mechanism}        

\author{Ziyang Lyu}


\institute{\\UNSW Data Since Hub, School of Mathematics and Statistics, University of New South
	Wales, Sydney, Australia\\
              \email{ziyang.lyu@unsw.edu.cn}           
}

\date{ }

\maketitle

\begin{abstract}
Semi-supervised learning (SSL) approaches have been successfully
applied in a wide range of engineering and scientific fields. This paper investigates the generative model framework with a missingness mechanism for unclassified observations, as introduced by \cite{ahfock2020apparent}.  
We show that in a partially classified sample, a classifier using Bayes’ rule of allocation with a missing-data mechanism can surpass a fully supervised classifier in a two-class normal homoscedastic model, especially with moderate to low overlap and proportion of missing class labels, or with large overlap but few missing labels.
It also outperforms a classifier with no missing-data mechanism regardless of
the overlap region or the proportion of missing class labels.
Our exploration of two- and three-component normal mixture models with unequal covariances through simulations further corroborates our findings. Finally,
we illustrate the use of the proposed classifier with a missing-data mechanism
on interneuronal and skin lesion datasets.
\end{abstract}

\section{Introduction\label{SEC:1}}
The field of machine learning has recently witnessed unprecedented success
in training extensive, deep neural networks, a feat largely attributable to the
availability of voluminous labeled datasets. However, producing these labels is often expensive and time-consuming, requiring specialized expertise. Semi-supervised learning (SSL), as detailed by \cite{chapelle2006semi}, presents a viable solution to this challenge.  A dataset suitable for SSL is composed of a labeled subset $\mathbb{D}^L=\left\{\left(\mathbf{y}_i, z_i\right) \mid i=1, \ldots, n_l\right\}$ and an unlabeled subset $\mathbb{D}^U=\left\{\mathbf{y}_{n+j} \mid j=1, \ldots, n_u\right\}$, where $\mathbf{y}_i$ represents a multi-dimensional observation, and $z_i$ denotes the class membership of $\mathbf{y}_i$. Principal approaches encompass generative models \cite{pan2006semi,kim2007texture,fujino2008semisupervised}, graph-based models \cite{blum1998combining,szummer2001partially,zhou2003learning}, and semi-supervised
support vector machines \cite{vapnik1998support,joachims1999transductive,lanckriet2004learning}.

Normal mixture models have become increasingly common in SSL. \cite{ahfock2023semi} provide an insightful literature review on SSL from a statistical perspective. 
Even though normal mixture models have been the subject of numerous
studies, such as those by \cite{kim2007texture,pan2006semi,come2009learning,szczurek2010introducing,huang2010semi}, these
works often make the critical assumption that the missing label process can
be overlooked for likelihood-based inference \cite{mclachlan1977estimating,chawla2005learning}. In contrast, our approach, following \cite{ahfock2020apparent}, treats the labels of unclassified observations as missing data and
proposes a framework for their missingness, drawing from Rubin’s seminal work on missingness in incomplete data analysis \cite{rubin1976inference}.  Here, the probability of
label missingness depends on a logistic model with a covariance equivalent to
an entropy-based measure. This dependency in the missingness pattern can
be harnessed to provide additional information about the optimal classifier as
prescribed by Bayes’ rule.

Within this framework, \cite{ahfock2020apparent} performed an asymptotic analysis showing that, in  two-class normal homoscedastic models, partially classified samples can provide more information than fully labeled ones.  { 	\cite{lyu2023gmmsslm} developed an R package focused on Gaussian mixture models incorporating a mechanism for missing data.}
However, their research did not clarify how to realize this effect nor did it expand to broader scenarios.
Consequently, our research seeks to identify the conditions where a classifier, based on a finite normal mixture model with a missing-data mechanism, surpasses a fully supervised classifier.  We build upon \cite{ahfock2020apparent}'s work to cover two- and three-class normal mixture models, assessed via simulation experiments. Moreover, we corroborate our results by analyzing two real-world datasets.

The structure of this paper is outlined as follows: Section~\ref{SEC:2} introduces the notation used in this paper. Section~\ref{SEC:3} details maximum likelihood estimation for completely and partially classified samples, considering both scenarios with and without the missing data mechanism. Section~\ref{SEC:4} examines the analysis of a two-class normal homoscedastic model. Simulation results for two- and three-class normal mixture models with unequal covariances are presented in Section~\ref{SEC:5}. Section~\ref{SEC:6} discusses the analysis of two datasets: the interneuronal and  skin lesion datasets. Finally, Section~\ref{SEC:7} provides a summary of our findings and concluding remarks.

\section{Notation\label{SEC:2}}
To define the notations, we let $\mathbf{y}_j$ denote the recorded $p$-dimensional observation, and we let $\mathbf{z}_j=\left(z_{1 j}, \ldots, z_{g j}\right)^T$ denote the corresponding $g$-dimensional vector of zero-one indicator variables defining the known class of origin of each, namely, $z_{i j}$ is equal to one if observation $\mathbf{y}_j$ belongs to class $i$ and is equal to zero otherwise, $(i=1, \ldots, g$ and $j=1, \ldots, n)$. We use $\mathbf{x}_{\mathrm{CC}}=\left(\mathbf{x}_1^T, \ldots, \mathbf{x}_n^T\right)^T$ with $\mathbf{x}_j=\left(\mathbf{y}_j^T, \mathbf{z}_j^T\right)^T$ to denote a set of completely classified samples. For a partially classified sample, in which some observations have no associated class labels, we let $m_j$ be the missing-label indicator for the observation $\mathbf{y}_j$, which is equal to 1 if $\mathbf{z}_j$ is missing and to zero if it is available $(j=1, \ldots, n)$. Accordingly, the partially classified sample $\mathbf{x}_{\mathrm{PC}}$ consists of those members $\mathbf{x}_j$ in $\mathbf{x}_{\mathrm{CC}}$ with $m_j=0$ and of only the observation $\mathbf{y}_j$ in $\mathbf{x}_{\mathrm{CC}}$ with $m_j=1$ (that is, the label $\mathbf{z}_j$ is missing). Henceforth, we will use the subscripts PC and CC to denote partially and completely classified samples, respectively.

A recorded $p$-dimensional observation $\mathbf{y}_j$ can arise from one of $g$ multivariate normal classes,
\begin{equation}\label{eq:2-1}
	\mathbf{Y}_j \sim N\left(\boldsymbol{\mu}_i, \boldsymbol{\Sigma}_i\right) \text { with probability } \pi_i \quad(i=1, \ldots, g ; j=1, \ldots, n) \text {, }
\end{equation}
where $\boldsymbol{\mu}_1, \ldots, \boldsymbol{\mu}_g$ denote the class means, and $\boldsymbol{\Sigma}_1, \ldots, \boldsymbol{\Sigma}_g$ denote the class covariance matrices. We denote the prior probabilities of the classes $C_1, \ldots, C_g$ by $\pi_1, \ldots, \pi_g$, respectively, where $1-\pi_g=\sum_{i=1}^{g-1} \pi_i$. We suppose that the density $f\left(\mathbf{y}_j\right)$ of $\mathbf{Y}_j$ can be written as
$$
f\left(\mathbf{y}_j\right)=\sum_{i=1}^g \pi_i f_i\left(\mathbf{y}_j ; \boldsymbol{\omega}_i\right),
$$
where $\boldsymbol{\omega}_i$ consists of the elements of $\boldsymbol{\mu}_i$ and the $\frac{1}{2} p(p+1)$ distinct elements of $\boldsymbol{\Sigma}_i$ $(i=1, \ldots, g)$. The vector $\boldsymbol{\theta}$ of all unknown parameters is given by $\boldsymbol{\theta}=$ $\left(\pi_1, \ldots, \pi_{g-1}, \boldsymbol{\omega}^T\right)^T$, where $\boldsymbol{\omega}=\left(\boldsymbol{\omega}_1^T, \ldots, \boldsymbol{\omega}_g^T\right)^T$. The classifier/Bayes' rule of allocation $R\left(\mathbf{y}_j ; \boldsymbol{\theta}\right)$ assigns an entity with observation $\mathbf{y}_j$ to class $C_k$ (that is, $R\left(\mathbf{y}_j ; \boldsymbol{\theta}\right)=k$, where $k$ represents the class membership) if
\begin{equation}\label{eq:2-2}
	k=\arg \max _i \tau_i\left(\mathbf{y}_j ; \boldsymbol{\theta}\right),
\end{equation}
where
\begin{equation}\label{eq:2-3}
	\tau_i\left(\mathbf{y}_j ; \boldsymbol{\theta}\right)=\frac{\pi_i f\left(\mathbf{y}_j ; \boldsymbol{\omega}_i\right)}{\sum_{h=1}^g \pi_h f\left(\mathbf{y}_j ; \boldsymbol{\omega}_i\right)}
\end{equation}
is the posterior probability that $\mathbf{Y}_j$ belongs to $C_i$ given $\mathbf{Y}_j=\mathbf{y}_j(i=1, \ldots, g$ and $j=1, \ldots, n)$.

The optimal (overall) error rate of the Bayes rule for a $g$-class normal
mixture model is given by
\begin{equation}\label{eq:2-4}
	\operatorname{err}(\mathbf{y} ; \boldsymbol{\theta})=1-\sum_{i=1}^g \pi_i \operatorname{Pr}\left\{R(\mathbf{y} ; \boldsymbol{\theta})=i \mid \mathbf{Z} \in C_i\right\}.
\end{equation}
The distribution of $R\left(\mathbf{y}_j ; \boldsymbol{\theta}\right)$  is very complicated, and manageable analytical
expressions were initially obtained only in special cases, such as those addressed by \cite{gilbert1969effect,han1969distribution,mclachlan1975some,hawkins1982extension}. Here, we write
\begin{equation}\label{eq:2-5}
	\operatorname{Pr}\left\{R(\mathbf{Y} ; \boldsymbol{\theta}) \in C_i \mid \mathbf{Z} \in C_i\right\}=\frac{\sum_{j=1}^n I_{C_i}\left(z_j\right) Q\left[z_j, R\left(\mathbf{y}_j ; \boldsymbol{\theta}\right)\right]}{\sum_{j=1}^n I_{C_i}\left(z_j\right)},
\end{equation}
where $Q[u, v]=1$ if $u=v$ and $Q[u, v]=0$ otherwise, and $I_{C_i}\left(z_j\right)$ is an indicator function for the $i$ th class.

\section{Methodology\label{SEC:3}}
We consider the ML estimation of the vector $\boldsymbol{\theta}$ containing all unknown parameters in the parametric families (\ref{eq:2-1}). For a partially classified sample $\mathbf{x}_{\mathrm{PC}}$, we let
\begin{equation}\label{eq:3-1}
	L_{\mathrm{PC}}^{(\mathrm{ig})}\left(\boldsymbol{\theta} ; \mathbf{x}_{\mathrm{PC}}\right)=\log L_C(\boldsymbol{\theta})+\log L_{U C}(\boldsymbol{\theta})
\end{equation}
denote the log partial likelihood function for $\boldsymbol{\theta}$ formed by ignoring the missing-data mechanism for the missing-label indicators $m_j$, where
$$
\begin{aligned}
	&\log L_C(\boldsymbol{\theta})=\sum_{j=1}^n\left(1-m_j\right) \sum_{i=1}^g z_{i j} \log \left\{\pi_i f\left(\mathbf{y}_j ; \boldsymbol{\omega}_i\right)\right\},\\
	\text{and} \quad&\log L_{U C}=\sum_{j=1}^n m_j \log \left\{\sum_{i=1}^g \pi_i f\left(\mathbf{y}_j ; \boldsymbol{\omega}_i\right)\right\},
\end{aligned}
$$
denote the $\log$ functions for $\boldsymbol{\theta}$ formed on the basis of the classified data and unclassified data, respectively. The $\log$ of the likelihood $L_{\mathrm{CC}}\left(\boldsymbol{\theta} ; \mathbf{x}_{\mathrm{CC}}\right)$ for the completely classified sample $\mathbf{x}_{\mathrm{CC}}$ is given by $\log L_C(\boldsymbol{\theta})$ with all $m_j=0$.
\begin{equation}\label{eq:3-2}
	\log L_{\mathrm{CC}}\left(\boldsymbol{\theta} ; \mathbf{x}_{\mathrm{CC}}\right)=\sum_{i=1}^g \sum_{j=1}^n z_{i j} \log \left\{\pi_i f\left(\mathbf{y}_j ; \boldsymbol{\omega}_i\right)\right\}.
\end{equation}

\subsection{Missing data mechanism}
 {In this study, we set aside the missing-data mechanism for likelihood inference in cases where missing labels are considered missing completely at random (MCAR), as defined in  Rubin (1976) . For an in-depth understanding of MCAR and its broader counterpart, missing at random (MAR), see \cite{mealli2015clarifying}. Under MCAR, the absence of labels is unrelated to both features and labels. In contrast, MAR allows for missingness dependent on features, but not on labels.
	As noted by \cite{mclachlan1989mixture}, in specific MAR instances, such as truncated features, ignoring missing data is justifiable. However, in our MAR context, disregarding missing data would be inappropriate. In fact, incorporating missing data can enhance the effectiveness of the classifier we aim to develop.
}

We introduce a random variable $M_j$ corresponding to the realized value $m_j$ of the missing-label indicator for the observation $\mathbf{y}_j$. The missing-data mechanism of \cite{rubin1976inference} is specified in the present context as
\begin{equation}\label{eq:3-3}
	\operatorname{Pr}\left\{M_j=1 \mid \mathbf{y}_j, \mathbf{z}_j\right\}=\operatorname{Pr}\left\{M_j=1 \mid \mathbf{y}_j\right\}=q\left(\mathbf{y}_j ; \boldsymbol{\theta}, \boldsymbol{\xi}\right) \quad(j=1, \ldots, n),
\end{equation}
where $\boldsymbol{\xi}=\left(\xi_0, \xi_1\right)^T$ is distinct from $\boldsymbol{\theta}$. The conditional probability $q\left(\mathbf{y}_j ; \boldsymbol{\theta}, \boldsymbol{\xi}\right)$ is taken to be a logistic function of the Shannon entropy $e n_j\left(\mathbf{y}_j ; \boldsymbol{\theta}\right)$, yielding
\begin{equation}\label{eq:3-4}
	q\left(\mathbf{y}_j ; \boldsymbol{\theta}, \boldsymbol{\xi}\right)=\frac{\exp \left\{\xi_0+\xi_1 \log e n_j\left(\mathbf{y}_j ; \boldsymbol{\theta}\right)\right\}}{1+\exp \left\{\xi_0+\xi_1 \log e n_j\left(\mathbf{y}_j ; \boldsymbol{\theta}\right)\right\}},
\end{equation}
where $e n_j\left(\mathbf{y}_j ; \boldsymbol{\theta}\right)=-\sum_{i=1}^g \tau_i\left(\mathbf{y}_j ; \boldsymbol{\theta}\right) \log \tau_i\left(\mathbf{y}_j ; \boldsymbol{\theta}\right)$. We use $\gamma$ to denote the expected proportion of missing class labels in the partially classified sample given by
\begin{equation}\label{eq:3-5}
	\gamma=\sum_{j=1}^n \mathrm{E}\left(M_j\right) / n=\mathrm{E}\left[\operatorname{Pr}\left\{M_j=1 \mid \mathbf{y}_j\right\}\right]=\mathrm{E}\{q(\mathbf{Y} ; \boldsymbol{\theta}, \boldsymbol{\xi})\}.
\end{equation}

We let $\boldsymbol{\Psi}=(\boldsymbol{\theta}^T, \boldsymbol{\xi}^T)^T$ be the vector of all unknown parameters. To construct the full likelihood function $L_{\mathrm{PC}}^{(\mathrm{full})}(\boldsymbol{\Psi})$ from the partially classified sample $\mathbf{x}_{\mathrm{PC}}$, we need expressions 
$$
f\left(\mathbf{y}_j, \mathbf{z}_j, m_j=0\right) \text { and } \quad f\left(\mathbf{y}_j, m_j=1\right),
$$
for the classified and unclassified observation $y_j$, respectively. Now, for a classified observation $\mathbf{y}_j$, it follows that
\begin{equation}\label{eq:3-6}
	\begin{aligned}
		f\left(\mathbf{y}_j, \mathbf{z}_j, m_j=0\right) & =f\left(\mathbf{z}_j\right) f\left(\mathbf{y}_y \mid \mathbf{z}_j\right) \operatorname{Pr}\left\{M_j=0 \mid \mathbf{y}_j, \mathbf{z}_j\right\} \\
		& =\prod_{i=1}^g\left\{\pi_i f_i\left(\mathbf{y}_j ; \boldsymbol{\omega}_i\right)\right\}^{z_{i j}}\left\{1-q\left(\mathbf{y}_j ; \boldsymbol{\theta}, \boldsymbol{\xi}\right)\right\},
	\end{aligned}
\end{equation}
while for an unclassified observation $\mathbf{y}_j$, we have that
\begin{equation}\label{eq:3-7}
	f\left(\mathbf{y}_j, m_j\!=\!1\right)\!=\!f\left(\mathbf{y}_j\right) \operatorname{Pr}\left\{M_j=1 \!\mid\! \mathbf{y}_j\right\}=\{\sum_{i=1}^g \pi_i f_i\left(\mathbf{y}_j ; \boldsymbol{\omega}_i\right)\} q\left(\mathbf{y}_j ; \boldsymbol{\theta}, \boldsymbol{\xi}\right).
\end{equation}
Then, the full likelihood function can be expressed as
\begin{equation}\label{eq:3-8}
	L_{\mathrm{PC}}^{(\mathrm{full})}(\boldsymbol{\Psi})=L_{\mathrm{PC}}\left(\boldsymbol{\theta}, \boldsymbol{\xi} ; \mathbf{x}_{\mathrm{PC}}\right)=\prod_{j=1}^n f^{\left(1-m_j\right)}\left(\mathbf{y}_j, \mathbf{z}_j, m_j=0\right) f^{m_j}\left(\mathbf{y}_j, m_j=1\right).
\end{equation}
Thus, we can write the full $\log$-likelihood function $\log L_{\mathrm{PC}}^{(\mathrm{full})}(\boldsymbol{\Psi})$ as 
\begin{equation}\label{eq:3-9}\log L_{\mathrm{PC}}^{(\mathrm{full})}\left(\boldsymbol{\Psi} ; \mathbf{x}_{\mathrm{PC}}\right)=\log L_{\mathrm{PC}}^{(\mathrm{ig})}\left(\boldsymbol{\theta} ; \mathbf{x}_{\mathrm{PC}}\right)+\log L_{\mathrm{PC}}^{(\mathrm{miss})}\left(\boldsymbol{\theta}, \boldsymbol{\xi} ; \mathbf{x}_{\mathrm{PC}}\right), \end{equation}where
\begin{equation}\label{eq:3-10}
	\log L_{\mathrm{PC}}^{(\mathrm{miss})}\left(\boldsymbol{\theta}, \boldsymbol{\xi} ; \mathbf{x}_{\mathrm{PC}}\right)=\sum_{j=1}^n\left[\left(1-m_j\right) \log \left\{1-q\left(\mathbf{y}_j ; \boldsymbol{\theta}, \boldsymbol{\xi}\right)\right\}+m_j \log q\left(\mathbf{y}_j ; \boldsymbol{\theta}, \boldsymbol{\xi}\right)\right]
\end{equation}
is the log-likelihood function formed on the basis of the missing-label indicator $m_j$.

We let $\boldsymbol{\theta}_{\mathrm{CC}}$,  $\boldsymbol{\theta}_{\mathrm{PC}}^{(\mathrm{ig})}$ and $\boldsymbol{\theta}_{\mathrm{PC}}^{(\mathrm{full})}$ be the estimates of $\boldsymbol{\theta}$ formed by consideration of the likelihood $L_{\mathrm{CC}}\left(\boldsymbol{\theta} ; \mathbf{x}_{\mathrm{CC}}\right)$,  $L_{\mathrm{PC}}^{(\mathrm{ig})}\left(\boldsymbol{\theta} ; \mathbf{x}_{\mathrm{PC}}\right)$ and $L_{\mathrm{PC}}^{(\mathrm{full})}\left(\boldsymbol{\theta}, \boldsymbol{\xi} ; \mathbf{x}_{\mathrm{PC}}\right)$, respectively. Similarly, we let $R(\hat{\boldsymbol{\theta}}_{\mathrm{CC}}), R(\hat{\boldsymbol{\theta}}_{\mathrm{PC}}^{(\mathrm{ig})})$ and $R(\hat{\boldsymbol{\theta}}_{\mathrm{PC}}^{(\mathrm{full})})$ denote the estimated Bayes' rule obtained by plugging in the estimates $\hat{\boldsymbol{\theta}}_{\mathrm{CC}}, \hat{\boldsymbol{\theta}}_{\mathrm{PC}}^{(\mathrm{ig})}$ and $\hat{\boldsymbol{\theta}}_{\mathrm{PC}}^{\text {(full)}}$, respectively. Henceforth, we abbreviate the above notation as $R$ with superscripts (ig) and (full) and subscripts CC and PC where appropriate.

\section{Two-class normal homoscedastic model\label{SEC:4}}
Given the complexity of the optimal error rate expression, providing a theoretical analysis of likelihood inference with a missing-data mechanism is challenging in a general case.
Thus, we study a particular case of $g$=2 classes under the assumption of equal-covariance $\left(\boldsymbol{\Sigma}_1=\boldsymbol{\Sigma}_2=\boldsymbol{\Sigma}\right)$. To simplify the numerical computations, \cite{ahfock2020apparent} gave a Taylor series approximation of the $\log$ entropy $\log \left\{e n_j\left(\mathbf{y}_j ; \boldsymbol{\theta}\right)\right\}$, shown as
$$
\log \left\{e n_j\left(\mathbf{y}_j ; \boldsymbol{\theta}\right)\right\}=-\left[\log \{\log (2)\}+d^2\left(\mathbf{y}_j ; \boldsymbol{\beta}\right) /\{8 \log (2)\}\right]+O\left(d^4\left(\mathbf{y}_j ; \boldsymbol{\beta}\right)\right),
$$
where
\begin{equation}\label{eq:4-1}
	d\left(\mathbf{y}_j ; \boldsymbol{\beta}\right)=\beta_0+\boldsymbol{\beta}_1^T \mathbf{y}_j,
\end{equation}
with
\begin{equation}\label{eq:4-2}
	\beta_0=\log \left(\pi_1 / \pi_2\right)-\frac{1}{2}\left(\boldsymbol{\mu}_1+\boldsymbol{\mu}_2\right)^T \boldsymbol{\Sigma}^{-1}\left(\boldsymbol{\mu}_1-\boldsymbol{\mu}_2\right) \quad \text { and } \quad \boldsymbol{\beta}_1=\boldsymbol{\Sigma}^{-1}\left(\boldsymbol{\mu}_1-\boldsymbol{\mu}_2\right).
\end{equation}

This expression provides a linear relationship between the negative log entropy $\log \left\{e n_j\left(\mathbf{y}_j ; \boldsymbol{\theta}\right)\right\}$ and the square of the discriminant function $d^2\left(\mathbf{y}_j ; \boldsymbol{\beta}\right)$, neither of which is large. Therefore, the negative log entropy in the conditional
probability (\ref{eq:3-4}), can be replaced by the square of the discriminant function
(\ref{eq:4-1}),
\begin{equation}\label{eq:4-3}
	q\left(\mathbf{y}_j ; \boldsymbol{\beta}, \boldsymbol{\xi}\right)=\frac{\exp \left\{\xi_0+\xi_1 d^2\left(\mathbf{y}_j ; \boldsymbol{\beta}\right)\right\}}{1+\exp \left\{\xi_0+\xi_1 d^2\left(\mathbf{y}_j ; \boldsymbol{\beta}\right)\right\}}.
\end{equation}

In convenient canonical form, let $\boldsymbol{\mu}_1=-\boldsymbol{\mu}_2=(\Delta / 2,0, \ldots, 0)^T$ and $\boldsymbol{\Sigma}=\mathbf{I}_p$, where $\Delta=[\left(\boldsymbol{\mu}_1-\boldsymbol{\mu}_2\right)^T \boldsymbol{\Sigma}^{-1}\left(\boldsymbol{\mu}_1-\boldsymbol{\mu}_2\right)]^{1 / 2}$ is the Mahalanobis distance between the two classes. Therefore, $\beta_0=\log \left(\pi_1 / \pi_2\right)$ and $\boldsymbol{\beta}_1=(\Delta, 0, \ldots, 0)^T$ in the discriminant function (\ref{eq:4-1}) .

\begin{lemma}\label{lem1}
	Given a two-class normal homoscedastic model in the canonical
	form in the same case of equal prior probabilities $\pi_1=\pi_2$, the entropy of
	observation $\mathbf{y}$ increases as the squared Mahalanobis distance decreases (equivalently the overlap region becomes smaller).
\end{lemma}
\begin{proof}
	In the canonical form, the discriminant function can be simplified as $d(\mathbf{y}, \Delta)=\Delta y_1$, where $y_1$ is the first element of $\mathbf{y}$. Under the two-class normal homoscedastic model, we can write the posterior probabilities as $\tau_1(\mathbf{y} ; \Delta)=$ $\exp \left(\Delta y_1\right) /\left\{1+\exp \left(\Delta y_1\right)\right\}$. Then, we obtain the entropy by
	$$
	e n(\mathbf{y} ; \Delta)=\log \left\{1+\exp \left(\Delta y_1\right)\right\}-\frac{\Delta y_1 \exp \left(\Delta y_1\right)}{1+\exp \left(\Delta y_1\right)} .
	$$
	Since the first derivative of $e n(\mathbf{y} ; \Delta)$ is given by
	$$
	e n^{\prime}(\mathbf{y} ; \Delta)=-\frac{\Delta y_1^2}{\left\{1+\exp \left(\left(\Delta y_1\right)\right)\right\}^2}<0,
	$$
	$e n(\mathbf{y} ; \Delta)$ is a monotonically decreasing function of $\Delta$.
\end{proof}
Since $\boldsymbol{\theta}$ corresponds to $\boldsymbol{\beta}$ as defined in (\ref{eq:4-2}), Bayes' rule $R(\hat{\boldsymbol{\theta}})$ is redefined as $R(\hat{\boldsymbol{\beta}})$ in terms of the estimates $\hat{\boldsymbol{\beta}}$. Likewise, the true error rate is also redefined as $\operatorname{err}(\boldsymbol{\beta})$.
	The asymptotic relative efficiencies (AREs) comparing the rule $R(\hat{\boldsymbol{\beta}}_{\mathrm{PC}}^{(\text {full) }})$ with $R(\hat{\boldsymbol{\beta}}_{\mathrm{CC}})$ and $R(\hat{\boldsymbol{\theta}}_{\mathrm{PC}}^{\mathrm{(ig})})$ are defined as
	$$
	\begin{aligned}
		&\operatorname{ARE}\left(\boldsymbol{\beta}_{\mathrm{PC}}^{(\mathrm{full})}, \boldsymbol{\beta}_{\mathrm{CC}}\right)=\frac{A E\left\{\operatorname{err}\left(\hat{\boldsymbol{\beta}}_{\mathrm{CC}}\right)\right\}-\operatorname{err}(\boldsymbol{\beta})}{A E\left\{\operatorname{err}\left(\hat{\boldsymbol{\beta}}_{\mathrm{PC}}^{(\mathrm{full})}\right)\right\}-\operatorname{err}(\boldsymbol{\beta})} \\
		\text { and }& \operatorname{ARE}\left(\boldsymbol{\beta}_{\mathrm{PC}}^{(\mathrm{full})}, \boldsymbol{\beta}_{\mathrm{PC}}^{(\mathrm{ig})}\right)=\frac{A E\left\{\operatorname{err}\left(\hat{\boldsymbol{\beta}}_{\mathrm{PC}}^{\text {(ig) }}\right)\right\}-\operatorname{err}(\boldsymbol{\beta})}{A E\left\{\operatorname{err}\left(\hat{\boldsymbol{\beta}}_{\mathrm{PC}}^{(\mathrm{full})}\right)\right\}-\operatorname{err}(\boldsymbol{\beta})},
	\end{aligned}
	$$
	respectively, where $A E\{\operatorname{err}(\hat{\boldsymbol{\beta}}_{\mathrm{CC}})\}, A E\{\operatorname{err}(\hat{\boldsymbol{\beta}}_{\mathrm{PC}}^{(\mathrm{ig})})\}$ and $A E\{\operatorname{err}(\hat{\boldsymbol{\beta}}_{\mathrm{PC}}^{(\mathrm{full})})\}$ denote the expansions of the expected error rates of $\mathrm{E}\{\operatorname{err}(\hat{\boldsymbol{\beta}}_{\mathrm{CC}})\}, \mathrm{E}\{\operatorname{err}(\hat{\boldsymbol{\beta}}_{\mathrm{PC}}^{(\mathrm{ig})})\}$ and $\mathrm{E}\{\operatorname{err}(\hat{\boldsymbol{\beta}}_{\mathrm{PC}}^{(\text {full })})\}$, respectively, up to and including terms of the first order in $1 / n$.
	
		Comparing the performance of Bayes' rule estimates shows the superiority of those from $\hat{\boldsymbol{\beta}}_{\mathrm{PC}}^{\text {(full)}}$ over $\hat{\boldsymbol{\beta}}_{\mathrm{CC}}$, especially when its ARE is 1 or greater.
	 Analogously, an enhanced performance can be observed in Bayes' rule estimations achieved using $\hat{\boldsymbol{\beta}}_{\mathrm{PC}}^{\text {(full)}}$ when compared with those using $\hat{\boldsymbol{\beta}}_{\mathrm{PC}}^{(\mathrm{ig})}$, under the condition that the ARE between $\boldsymbol{\beta}_{\mathrm{PC}}^{(\text {full)}}$ and $\boldsymbol{\beta}_{\mathrm{PC}}^{(\mathrm{ig})}$ is at least 1. As demonstrated by \cite{ahfock2020apparent} in their second theorem, the ARE of $\boldsymbol{\beta}_{\mathrm{PC}}^{(\text{full)}}$ and $\boldsymbol{\beta}_{\mathrm{CC}}$ was computed under the circumstance of equivalent prior probabilities $\pi_1=\pi_2$. This calculation depends on the variables $\Delta$ and $\boldsymbol{\xi}$.
	  Simulations further showed cases where the ARE of $\boldsymbol{\beta}_{\mathrm{PC}}^{(\mathrm{full})}$ to $\boldsymbol{\beta}_{\mathrm{CC}}$ was 1 or higher. This implies an improved performance of Bayes' rule estimations achieved via $\hat{\boldsymbol{\beta}}_{\mathrm{PC}}^{\text {(full)}}$ in comparison to those obtained through $\hat{\boldsymbol{\beta}}_{\mathrm{CC}}$. In the following, we elucidate the expression for the ARE between $\boldsymbol{\beta}_{\mathrm{PC}}^{\text {(full)}}$ and $\boldsymbol{\beta}_{\mathrm{PC}}^{(\mathrm{ig})}$ under the equivalent scenario of equal prior probabilities, namely $\pi_1=\pi_2$.

\begin{theorem}\label{thm1}
Under the missing-label model defined in \ref{eq:4-3}, the ARE of $R(\hat{\boldsymbol{\beta}}_{PC}^{(full)})$ compared to $R(\hat{\boldsymbol{\beta}}_{P C}^{(ig)})$ in the case of $\pi_1=\pi_2$ is given by

\begin{equation}\label{eq:4-4}
	ARE\left(\boldsymbol{\beta}_{PC}^{(full)}, \boldsymbol{\beta}_{PC}^{(ig)}\right)=\frac{\left(4+\Delta^2\right)-\gamma d_0+b_0}{\left(4+\Delta^2\right)^{-1}-c_0},
\end{equation}
for all $p$, where

$$
\begin{aligned}
	b_0 & =\int_{-\infty}^{\infty} 4 \xi_1^2 \Delta^2 y_1^2 q_1\left(y_1\right)\left(1-q\left(y_1\right)\right) f_{y_1}\left(y_1\right) d y_1, \\
	c_0 & =\int_{-\infty}^{\infty} \frac{\phi\left(y_1\right) \exp \left\{-\Delta^2 / 8\right\}}{1 / 2 \exp \left\{\Delta y_1 / 2\right\}+1 / 2 \exp \left\{-\Delta y_1 / 2\right\}} d y_1, \\
	d_0 & =\int_{-\infty}^{\infty} \tau_1\left(y_1\right) \tau_2\left(y_1\right) q_1\left(y_1\right) \gamma^{-1} f_{y_1}\left(y_1\right) d y_1,
\end{aligned}
$$
with

$$
\begin{aligned}
	& \tau_i\left(y_1\right)=\operatorname{Pr}\left(Z_i=1 \mid(\mathbf{Y})_1=y_1\right) \quad(i=1,2), \\
	& q_1\left(y_1 ; \Delta, \boldsymbol{\xi}\right)=\operatorname{Pr}\left\{M=1 \mid(\mathbf{Y})_1=y_1\right\}, \\
	& f_{y_1}\left(y_1 ; \Delta, \pi_1\right)=\pi_1 \phi\left(y_1 ; \Delta / 2,1\right)+\left(1-\pi_1\right) \phi\left(y_1 ;-\Delta / 2,1\right).
\end{aligned}
$$
\end{theorem}

\begin{proof}
	\cite{efron1975efficiency} gave the first order expansion of the expected excess error rate of the plug-in form of Bayes' rule using the estimator $\hat{\boldsymbol{\beta}}$ of $\boldsymbol{\beta}$, where $\sqrt{n}(\hat{\boldsymbol{\beta}}-\boldsymbol{\beta})$ converges in distribution to $N(0, \mathbf{V})$, as $n \rightarrow \infty$, and where $\mathbf{V}$ is a $(p+1) \times(p+1)$ variance matrix of $\sqrt{n}(\hat{\boldsymbol{\beta}}-\boldsymbol{\beta})$, the first and second order moments also converge. The expectation of the so-called excess error rate can be expanded as follows:
	
	$$
	\mathrm{E}\{\operatorname{err}(\hat{\boldsymbol{\beta}})\}-\operatorname{err}(\boldsymbol{\beta})=\frac{\pi_1 \phi\left(\Delta^* ; 0,1\right)}{2 \Delta n} w+o(1 / n),
	$$
	where $\Delta^*=\frac{1}{2} \Delta-\lambda / \Delta$ with $\lambda=\log \left(\pi_1 / \pi_2\right), \phi\left(y ; \mu, \sigma^2\right)$, denotes the normal density with mean $\mu$ and variance $\sigma^2$, and
	$$
	w=v_{00}-\frac{2 \lambda}{\Delta} v_{01}+\frac{\lambda^2}{\Delta^2} v_{11}+\sum_{i=2}^p v_{i i},
	$$
	with $v_{i j}$ representing the element in the $i$ th row and $j$ th column of $\mathbf{V}$.

\end{proof}

It has been shown \cite{ahfock2020apparent} that the information matrix $\mathbf{I}_{\mathrm{PC}}^{(\mathrm{full})}(\boldsymbol{\beta})$ can be decomposed as follow
$$
\mathbf{I}_{\mathrm{PC}}^{(\mathrm{full})}(\boldsymbol{\beta})=\mathbf{I}_{\mathrm{CC}}(\boldsymbol{\beta})-\gamma \mathbf{I}_{\mathrm{CC}}^{(\mathrm{clr})}(\boldsymbol{\beta})+\mathbf{I}_{\mathrm{PC}}^{(\mathrm{miss})}(\boldsymbol{\beta}),
$$
where $\gamma$ is the expected proportion of missing class labels in the partially classified sample, $\mathbf{I}_{\mathrm{CC}}(\boldsymbol{\beta})$ is the information about $\boldsymbol{\beta}$ contained in $\mathbf{x}_{\mathrm{CC}}, \mathbf{I}_{\mathrm{CC}}^{(\mathrm{clr})}(\boldsymbol{\beta})$ is the conditional information about $\boldsymbol{\beta}$ under the logistic regression model for the distribution of the class labels given $\mathbf{x}_{\mathrm{CC}}$, and $\mathbf{I}_{\mathrm{PC}}^{\text {(miss) }}(\boldsymbol{\beta})$ is the information about $\boldsymbol{\beta}$ contained in the missing-label indicators under the assumed logistic model for their distribution given $\mathbf{x}_{\mathrm{PC}}$. In the case of $\pi_1=\pi_2$, the information matrix is a diagonal matrix, $\mathbf{I}_{\mathrm{PC}}^{(\mathrm{full})}(\boldsymbol{\beta})=\operatorname{diag}\left\{u_0, u_1 . u_0, \ldots, u_0\right\}$; we have $\lambda=0$ and $w=v_{00}+\sum_{i=2}^p v_{i i}$. Since the asymptotic covariance matrix $\mathbf{V}$ is given by $n \times\{\mathbf{I}(\boldsymbol{\beta})\}^{-1}$, we obtain

$$
\mathrm{E}\left\{\operatorname{err}(\hat{\boldsymbol{\beta}}_{\mathrm{PC}}^{(\mathrm{full})})\right\}-\operatorname{err}(\boldsymbol{\beta})=\frac{\pi_1 p \phi(\Delta / 2 ; 0,1)}{2 \Delta u_0}+o(1 / n),
$$
where

$$
\begin{aligned}
	& u_0=\left(4+\Delta^2\right)^{-1}-\gamma d_0+b_0, \\
	& b_0=\int_{-\infty}^{\infty} 4 \xi_1^2\left(\Delta^2 y_1^2+2 \lambda \Delta+\lambda^2\right) q_1\left(y_1\right)\left(1-q\left(y_1\right)\right) f_{y_1}\left(y_1\right) d y_1, \\
	& d_0=\int_{-\infty}^{\infty} \tau_1\left(y_1\right) \tau_2\left(y_1\right) q_1\left(y_1\right) \gamma^{-1} f_{y_1}\left(y_1\right) d y_1 .
\end{aligned}
$$
\cite{o1978normal} assumed that the missingness of the labels did not depend on the data and showed that the information matrix $\mathbf{I}_{\mathrm{PC}}^{(\mathrm{ig})}(\boldsymbol{\beta})$ can be decomposed as follows
$$
\mathbf{I}_{\mathrm{PC}}^{(\mathrm{ig})}(\boldsymbol{\beta})=\mathbf{I}_{\mathrm{CC}}(\boldsymbol{\beta})-\gamma \mathbf{I}_{\mathrm{CC}}^{(\mathrm{lr})}(\boldsymbol{\beta}),
$$
where $\mathbf{I}_{\mathrm{CC}}^{(\mathrm{lr})}(\boldsymbol{\beta})$ is the information about $\boldsymbol{\beta}$ contained in the inverse of the asymptotic covariance matrix of the $\boldsymbol{\beta}$ obtained from the marginal likelihood. For $\pi_1=\pi_2, \mathbf{I}_{\mathrm{PC}}^{(\mathrm{ig})}(\boldsymbol{\beta})=\operatorname{diag}\left\{a_0, a_1, a_0, \ldots, a_0\right\}$. Similarly, we obtain
$$
\mathrm{E}\left\{\operatorname{err}(\hat{\boldsymbol{\beta}}_{\mathrm{PC}}^{(\mathrm{ig})})\right\}-\operatorname{err}(\boldsymbol{\beta})=\frac{\pi_1 p \phi(\Delta / 2 ; 0,1)}{2 \Delta a_0}+o(1 / n),
$$
where
$$
a_0=\left(4+\Delta^2\right)^{-1}-c_0, \quad \text { and } \quad c_0=\int_{-\infty}^{\infty} \frac{\phi\left(y_1\right) \exp \left\{-\Delta^2 / 8\right\}}{1 / 2 \exp \left\{\Delta y_1 / 2\right\}+1 / 2 \exp \left\{-\Delta y_1 / 2\right\}} d y_1.
$$
Our examination of the two-class normal homoscedastic model, even with equal prior probabilities $(\pi_1=\pi_2)$, unveils a substantial complexity in the expressions for $\operatorname{ARE}(\boldsymbol{\beta}_{\mathrm{PC}}^{\text {(full) }}, \boldsymbol{\beta}_{\mathrm{CC}})$ and $\operatorname{ARE}(\boldsymbol{\beta}_{\mathrm{PC}}^{(\mathrm{full})}, \boldsymbol{\beta}_{\mathrm{PC}}^{(\mathrm{ig})})$. Consequently, elucidating the relationship among $\operatorname{ARE}(R(\hat{\boldsymbol{\theta}}_{\mathrm{PC}}^{\text {(full) }})), \Delta$, and $\boldsymbol{\xi}$ remains an intricate task. In an attempt to demystify this complexity, we take the approach of fixing $\xi_0$ and $\xi_1$, then systematically exploring the range of $\Delta$ values that yield $\operatorname{ARE}(\boldsymbol{\beta}_{\mathrm{PC}}^{(\mathrm{full})}, \boldsymbol{\beta}_{\mathrm{CC}}) \geq 1$ and $\operatorname{ARE}(\boldsymbol{\beta}_{\mathrm{PC}}^{(\text {full })}, \boldsymbol{\beta}_{\mathrm{PC}}^{(\mathrm{ig})}) \geq 1$. Figures~\ref{fig:1} and \ref{fig:2} present the interrelation between $\mathrm{ARE}$ and $\Delta$ under assorted combinations of $\xi_0 \in\{0.1,0.5,1\}$ and $\xi_1 \in\{-0.1,-0.5,-1,-3,-5\}$, given the prior probability equalities, $\pi_1=\pi_2$. Even for the cases with unequal prior probabilities, specifically $\pi_1 \in\{0.1,0.2,0.3,0.4\}$, we observed trends that are analogous to those displayed in Figures~\ref{fig:1} and \ref{fig:2}. The findings from this investigation are elaborated in the supplementary.

As seen in Figure~\ref{fig:1}, as $\Delta$ increases under different combinations of $\xi_0$ and $\xi_1$, the value of $\operatorname{ARE}(\boldsymbol{\beta}_{\mathrm{PC}}^{\text {(full)}}, \boldsymbol{\beta}_{\mathrm{CC}})$ initially increases from 0 to a maximum and then decreases and converges to 1 . This behavior can be interpreted as follows
\begin{itemize}
	\item When the squared Mahalanobis distance between the two classes is too large (a small overlap region, e.g., Figure~\ref{fig:3}c), the value of $\operatorname{ARE}(\boldsymbol{\beta}_{\mathrm{PC}}^{(\mathrm{full})}, \boldsymbol{\beta}_{\mathrm{CC}})$ converges to 1, so there is no difference between the estimated Bayes' rules $R(\hat{\boldsymbol{\beta}}_{\mathrm{PC}}^{(\text {full })})$ and $R(\hat{\boldsymbol{\beta}}_{\mathrm{CC}})$.
	
	\item When the squared Mahalanobis distance between the two classes is too small (a large overlap region, e.g., Figure~\ref{fig:3}a), the value of $\mathrm{ARE}(\boldsymbol{\beta}_{\mathrm{PC}}^{(\mathrm{full})}, \boldsymbol{\beta}_{\mathrm{CC}})$ is smaller than 1 , indicating that $R(\hat{\boldsymbol{\beta}}_{\mathrm{PC}}^{(\text {full) }})$ performs worse than $R(\hat{\boldsymbol{\beta}}_{\mathrm{CC}})$;
	\item When the squared Mahalanobis distance between the two classes is moderate (a moderate overlap region, e.g., Figure~\ref{fig:3}b), the value of $\operatorname{ARE}(\boldsymbol{\beta}_{\mathrm{PC}}^{\text {(full) }}, \boldsymbol{\beta}_{\mathrm{CC}})$ is greater than 1, which means that $R(\hat{\boldsymbol{\beta}}_{\mathrm{PC}}^{(\mathrm{full})})$ performs better than $R(\hat{\boldsymbol{\beta}}_{\mathrm{CC}})$.
\end{itemize}
Moreover, when we focus on a partially classified sample with a large overlap region (such as $\Delta=0.25$ ), we find that when $\xi_0$ is fixed, the smaller the $\xi_1$ value is, the larger the ARE value will be, implying that when $\Delta$ is fixed, a reduction in the proportion of missing class labels will cause the error rate of $R(\hat{\boldsymbol{\theta}}_{\mathrm{PC}}^{(\mathrm{full})})$ to be close to or even better than that of $R(\hat{\boldsymbol{\beta}}_{\mathrm{CC}})$.

Figure~\ref{fig:2} shows that as $\Delta$ increases under different combinations of $\xi_0$ and $\xi_1$, the value of $\operatorname{ARE}(\boldsymbol{\beta}_{\mathrm{PC}}^{(\mathrm{full})}, \boldsymbol{\beta}_{\mathrm{PC}}^{(\mathrm{ig})})$ initially increases from 1 to a maximum and then decreases and converges to 1, indicating that the performance of $R(\boldsymbol{\beta}_{\mathrm{PC}}^{(\mathrm{full})})$ is no worse than that of $R(\boldsymbol{\beta}_{\mathrm{PC}}^{(\mathrm{ig})})$.

\begin{figure}[t]
	\includegraphics[width=\linewidth]{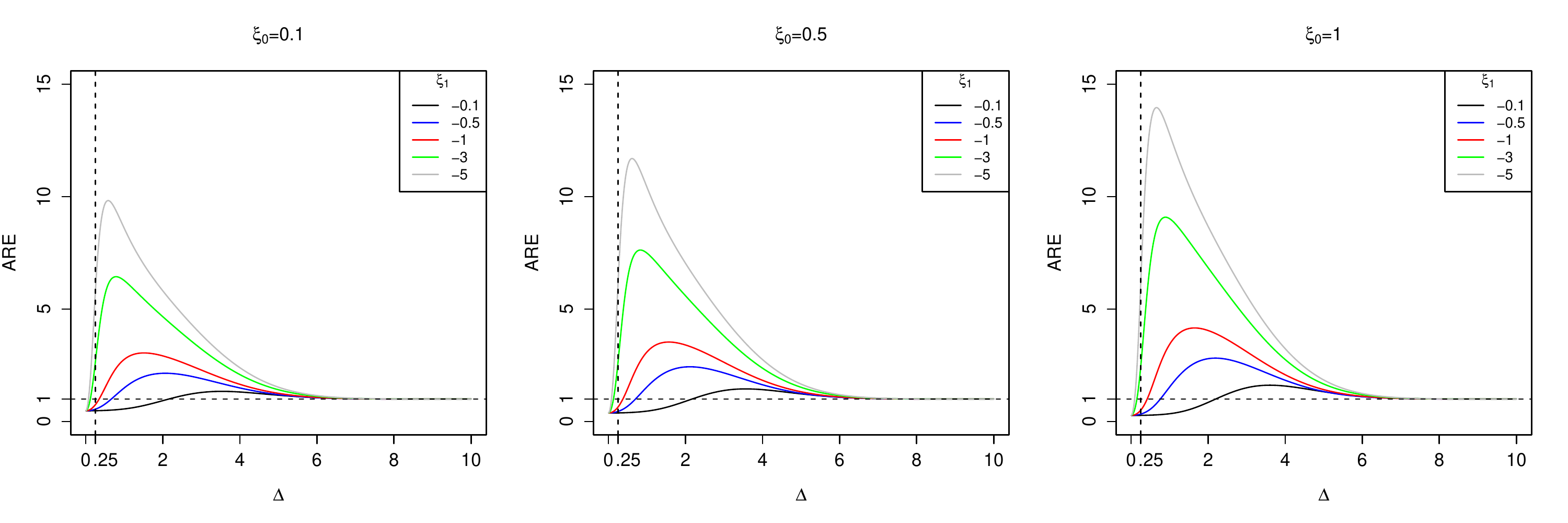}
	\caption{Plot of the asymptotic relative efficiency $\operatorname{ARE}(\boldsymbol{\beta}_{\mathrm{PC}}^{(\mathrm{full})}, \boldsymbol{\beta}_{\mathrm{CC}})$ versus the squared root of the squared Mahalanobis distance between the two classes, $\Delta$, for $\pi_1=\pi_2$.}
	\label{fig:1}
\end{figure}
\begin{figure}[t]
	\includegraphics[width=\linewidth]{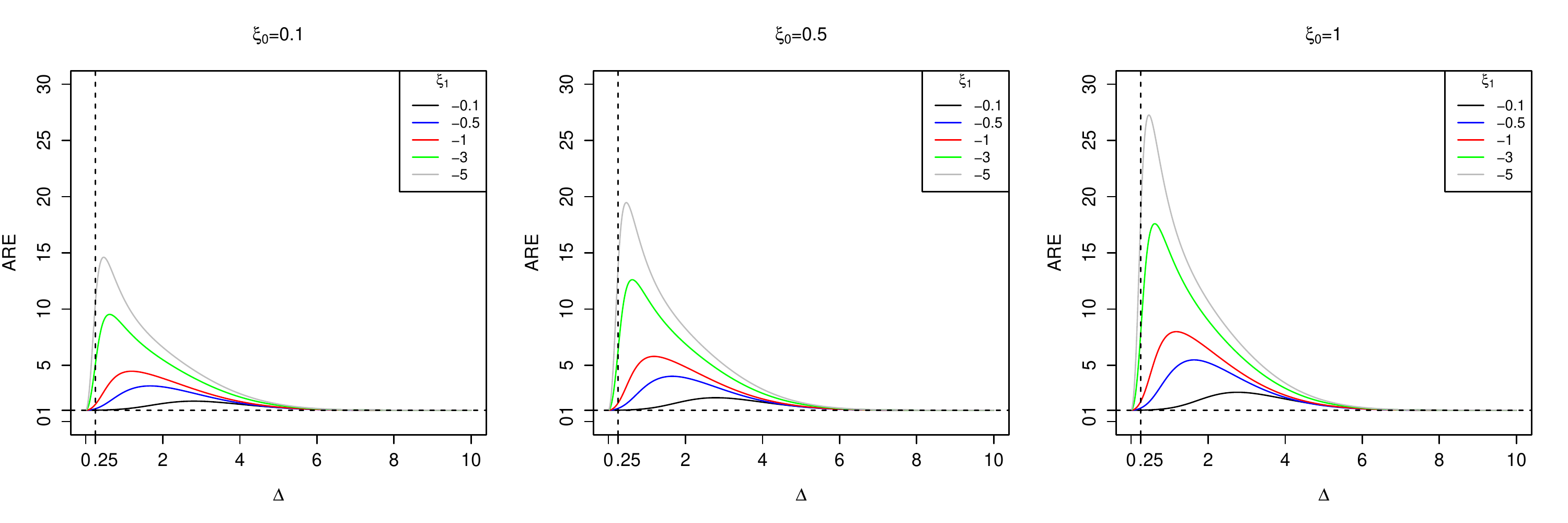}
	\caption{Plot of the asymptotic relative efficiency $\operatorname{ARE}(\boldsymbol{\beta}_{\mathrm{PC}}^{(\mathrm{full})}, \boldsymbol{\beta}_{\mathrm{PC}}^{(\mathrm{ig})})$ versus the squared root of the squared Mahalanobis distance between the two classes, $\Delta$, for $\pi_1=\pi_2$.}
	\label{fig:2}
\end{figure}

\begin{itemize}
	\item When the squared Mahalanobis distance between the two classes is too large or too small (small or large overlap region, e.g., Figures~\ref{fig:3}a and \ref{fig:3}c), the value of $\operatorname{ARE}(\boldsymbol{\beta}_{\mathrm{PC}}^{(\mathrm{full})}, \boldsymbol{\beta}_{\mathrm{PC}}^{(\mathrm{ig})})$ converges to 1, so there is no difference between the estimated Bayes' rules $R(\hat{\boldsymbol{\beta}}_{\mathrm{PC}}^{(\mathrm{full})})$ and $R(\hat{\boldsymbol{\beta}}_{\mathrm{PC}}^{(\mathrm{ig})})$. 
	\item When the squared Mahalanobis distance between the two classes is moderate, (a moderate overlap region, e.g., Figure~\ref{fig:3}b), the value of $\operatorname{ARE}(\boldsymbol{\beta}_{\mathrm{PC}}^{\text {(full) }}, \boldsymbol{\beta}_{\mathrm{PC}}^{(\mathrm{ig})})$ is greater than 1 , which means that $R(\hat{\boldsymbol{\beta}}_{\mathrm{PC}}^{(\mathrm{full})})$ performs better than $R(\hat{\boldsymbol{\beta}}_{\mathrm{PC}}^{(\mathrm{ig})})$.
\end{itemize}

Moreover, when we focus on a partially classified sample with a large overlap region (such as $\Delta=0.25$ ), we find that when $\xi_0$ is fixed, the smaller the $\xi_1$ value is, the larger the ARE value will be, implying that when $\Delta$ is fixed, a reduction in the proportion of missing class label will cause the error rate of $R(\hat{\boldsymbol{\theta}}_{\mathrm{PC}}^{(\mathrm{full})})$ to be much smaller than that of $R(\hat{\boldsymbol{\beta}}_{\mathrm{PC}}^{(\mathrm{ig})})$. 
\begin{figure}[h!]
	\subfigure[]{\includegraphics[width=0.3\linewidth]{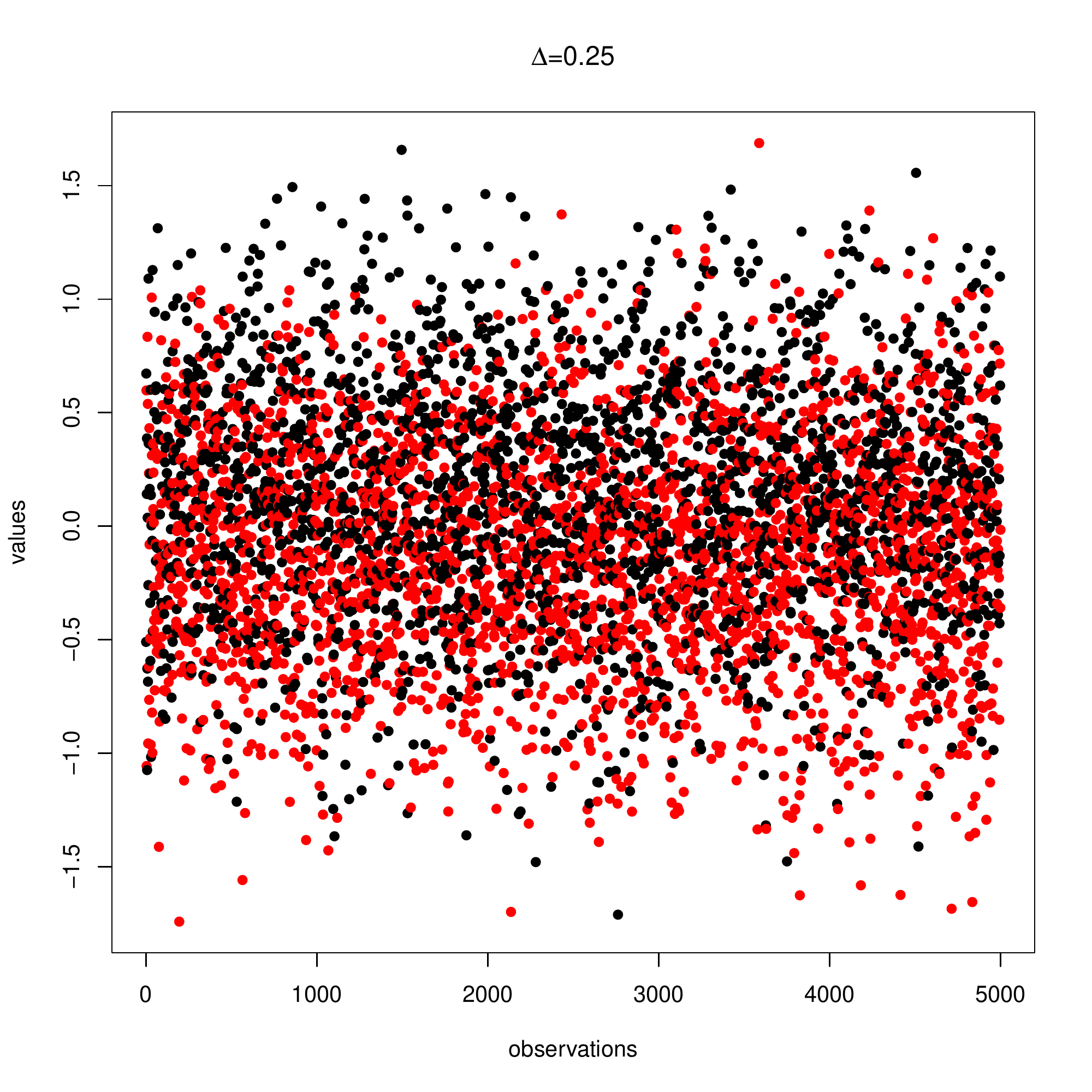}}
	\subfigure[]{\includegraphics[width=0.3\linewidth]{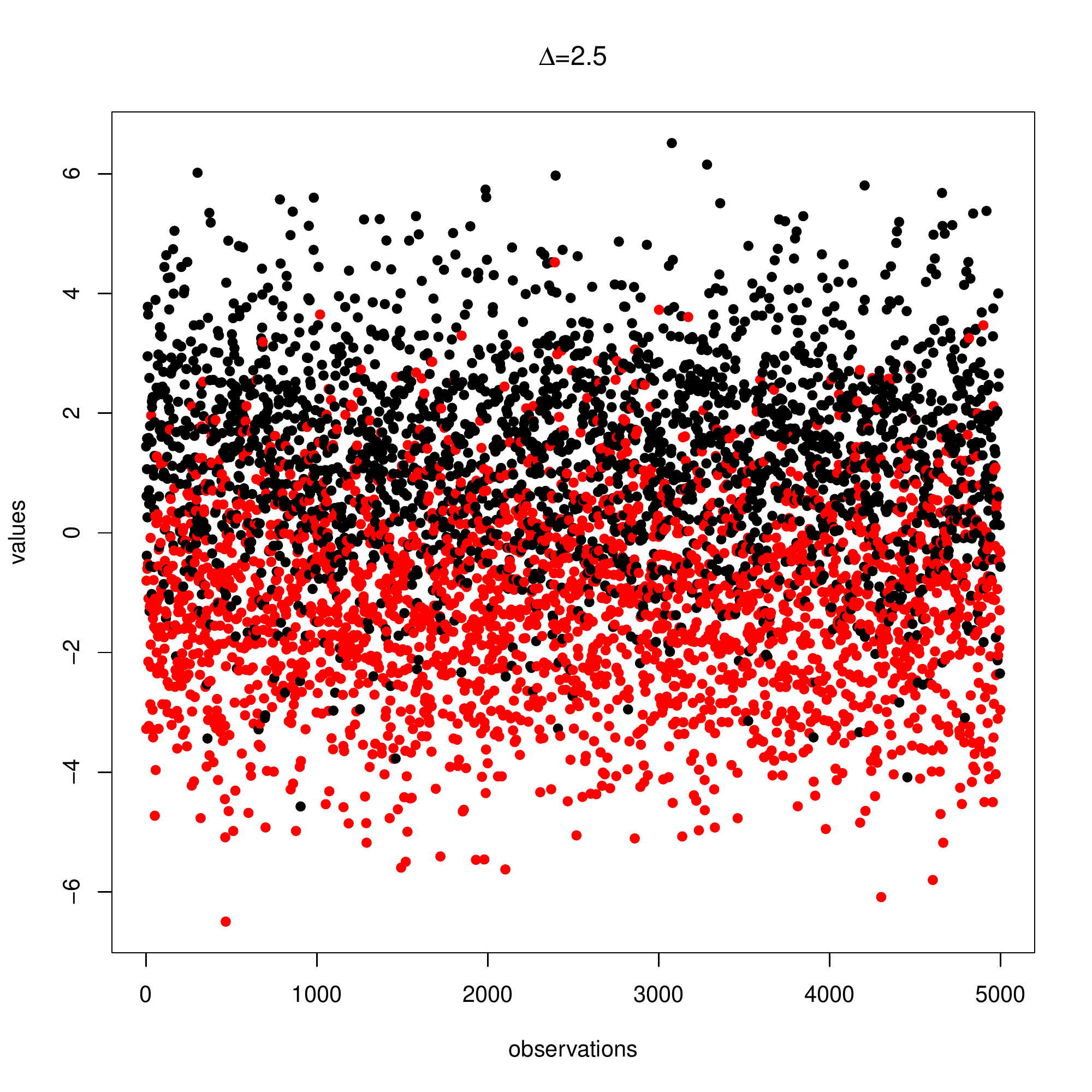}}
	\subfigure[]{\includegraphics[width=0.3\linewidth]{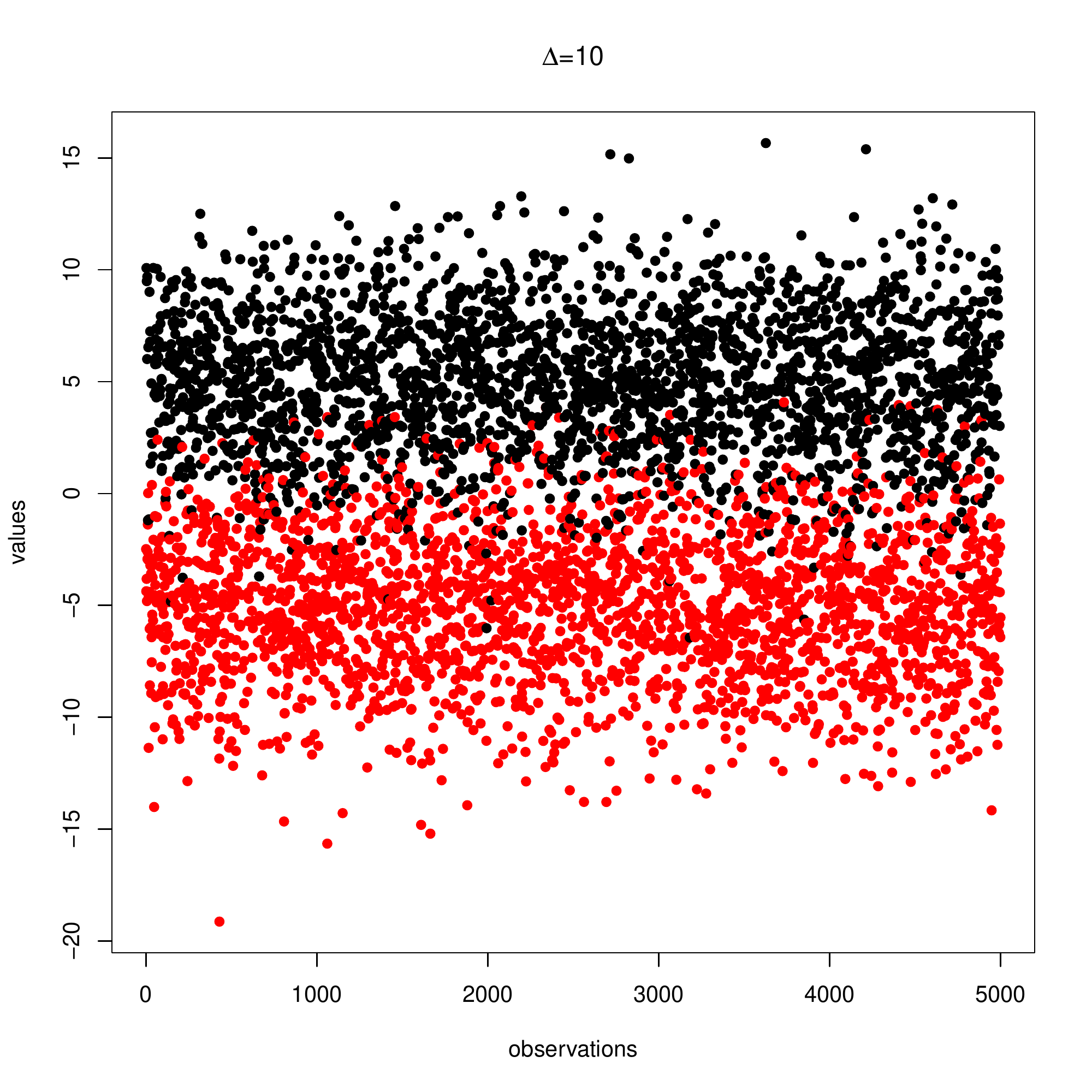}}
	\subfigure[]{\includegraphics[width=0.3\linewidth]{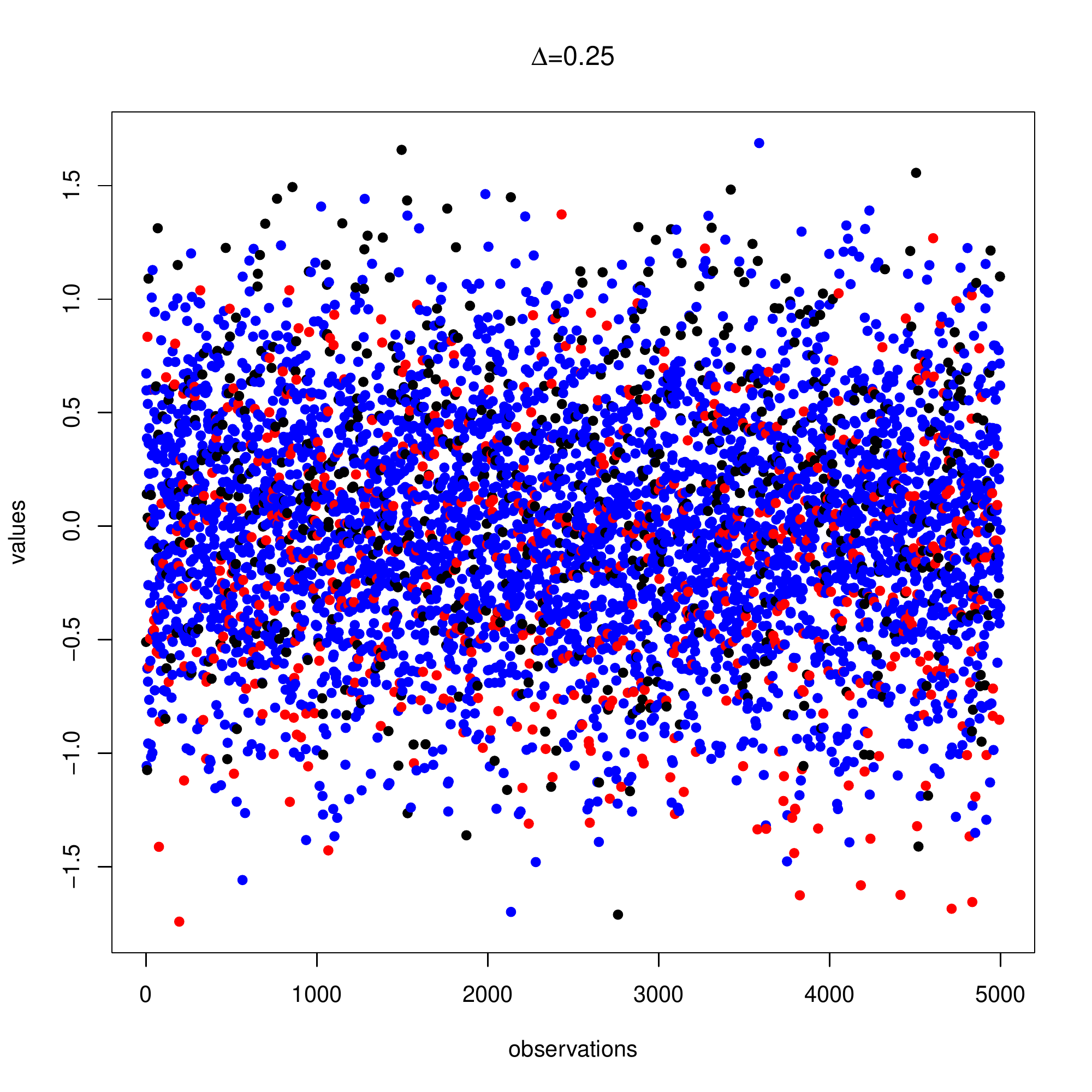}}\hspace{0.5cm}
	\subfigure[]{\includegraphics[width=0.3\linewidth]{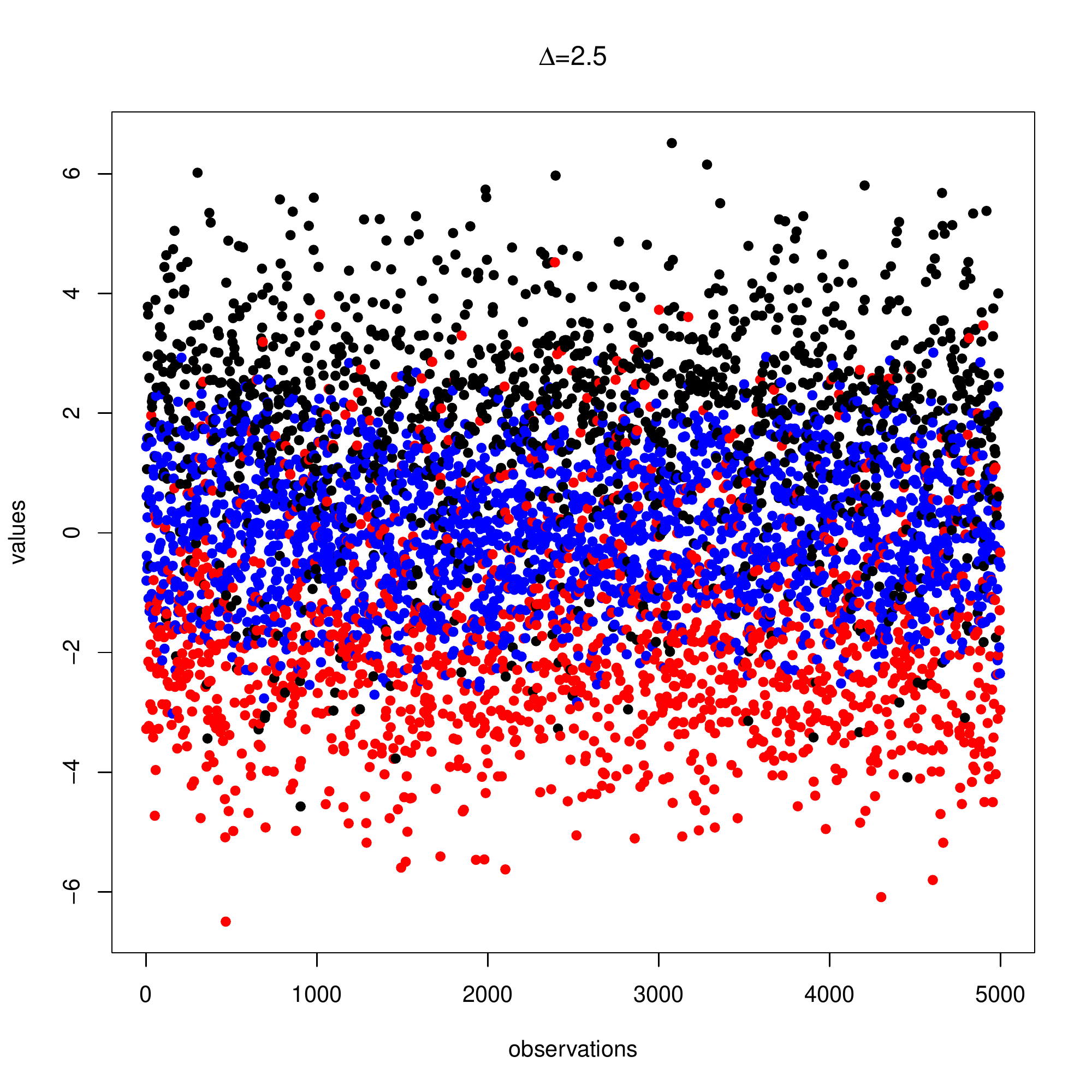}}\hspace{0.5cm}
	\subfigure[]{\includegraphics[width=0.3\linewidth]{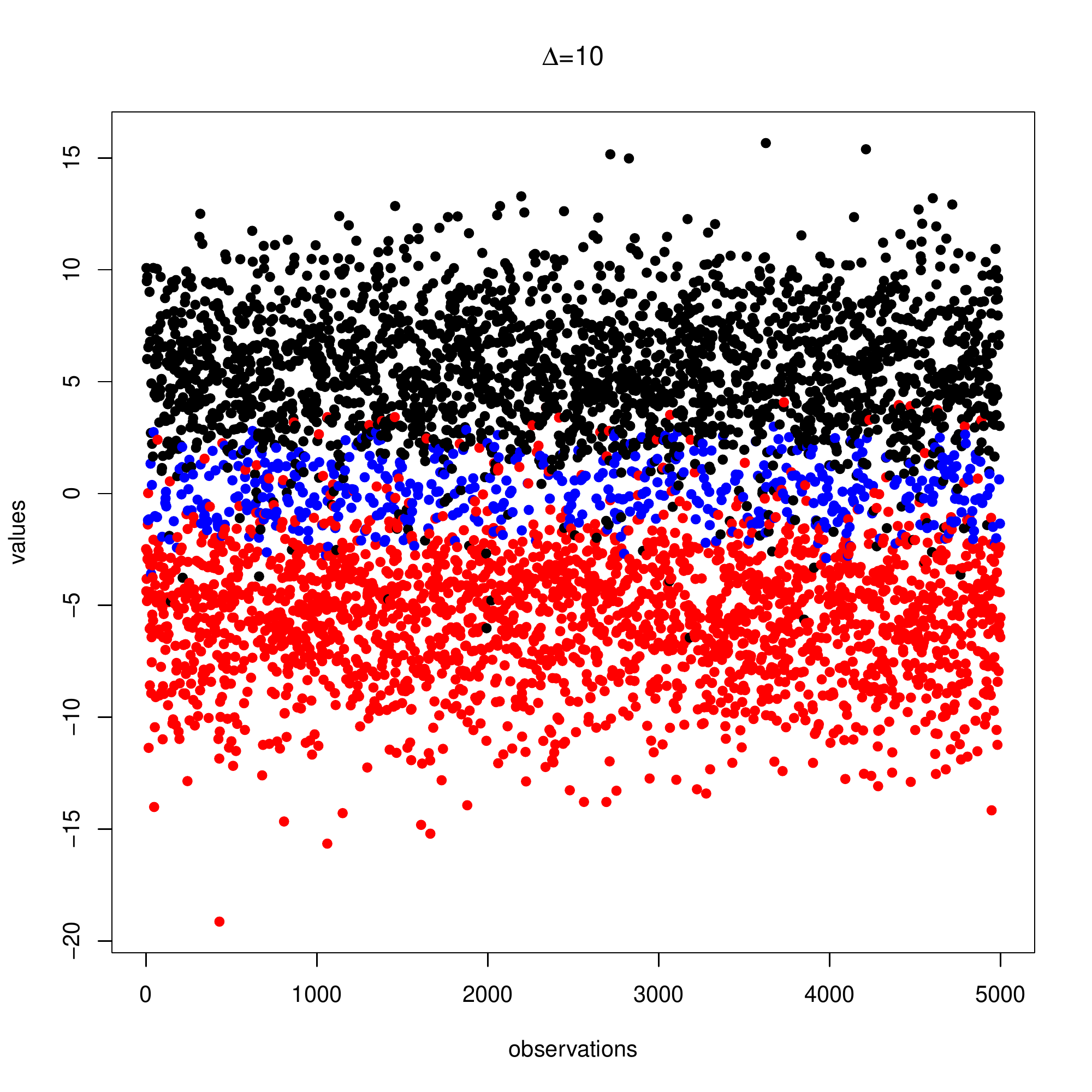}}
	\caption{Simulated scalar observations in the case of $n=5000, \Delta \in$ $\{0.25,2.5,10\}$, and $\pi_1=\pi_2$ (black represents class 1 ;red represents class 2 ; blue represents unclassified observations with $\left.\xi=(1,0.5)^T\right)$. (a) Large overlap region, (b) Moderate overlap region, and (c) Small overlap region represent the simulated scalar observations when $\pi_1=\pi_2$. (d) to (f) show the same overlap regions as (a) to (c) but include unclassified observations (blue) represented by $\boldsymbol{\xi}=(1,0.5)^T$.}
	\label{fig:3}
\end{figure}

When $\boldsymbol{\xi}$ is fixed, equation (\ref{eq:4-3}) is a monotonically decreasing function of $\Delta$, meaning that the proportion of missing class labels in the partially classified sample decreases as $\Delta$ increases with fixed $\boldsymbol{\xi}$. In other words, when $\boldsymbol{\xi}$ is fixed, large/moderate/small overlap leads to a relatively high/median/low proportion of missing class labels under the missing-label mechanism (e.g., Figures~\ref{fig:3}d, \ref{fig:3}e and \ref{fig:3}f). This model is adopted with the understanding that in many datasets, the unclassified observations tend to fall in regions of overlap
between the classes in the feature space.

\noindent\textbf{Remark 1}
	Within the framework of the missing-data mechanism, we prefer to apply the estimated Bayes' rule obtained using the ML estimator $\hat{\boldsymbol{\beta}}_{\mathrm{PC}}^{\text {(full) }}$ instead of $\hat{\boldsymbol{\beta}}_{\mathrm{PC}}^{(\mathrm{ig})}$ for a partially classified sample, regardless of the extent of overlap or the proportion of missing class labels. For a completely classified sample, the estimated Bayes' rule obtained using the ML estimator $\hat{\boldsymbol{\beta}}_{\mathrm{PC}}^{(\mathrm{full})}$ is preferred over that obtained using $\hat{\boldsymbol{\beta}}_{\mathrm{CC}}$ either when both the overlap region and the proportion of missing class labels are moderate or small, or when the overlap region is large but the proportion of missing class labels is relatively small.

\section{Simulation study\label{SEC:5}}
To our knowledge, the ARE analysis is limited to the specific case of a two-class normal homoscedastic model because of the calculation's complexity. 
Exploring more general situations is challenging as it requires theoretical discernment of the conditions where Bayes' rule, estimated using a full likelihood ML estimator, is valid. Therefore, we broaden our investigation through simulation studies to evaluate whether our preliminary conclusions hold in more varied contexts.

\subsection{Two-class normal model with unequal covariance matrices\label{SEC:5-1}}
We continue to study the two-class normal model but extend it to unequal covariance matrices. We consider proportional covariance matrices; i.e., we assume $\boldsymbol{\Sigma}_1=$ $\boldsymbol{\Sigma}$ and $\boldsymbol{\Sigma}_2=\lambda \boldsymbol{\Sigma}$. This model has been studied by \cite{bartlett1963discrimination,gilbert1969effect,han1969distribution,marks1974discriminant,mclachlan1975iterative}. The error rate of the Bayes' rule is evaluated for various combinations of the parameters in canonical form, where $\boldsymbol{\mu}_1=\mathbf{0}, \boldsymbol{\mu}_2=(\Delta, 0, \ldots, 0)^T$, and $\boldsymbol{\Sigma}=\mathbf{I}$. \cite{marks1974discriminant,mclachlan1975iterative} introduced a parameter $D=\Delta /(1+\sqrt{\lambda})$ as a measure of the degree of separation of two populations (when $\lambda=1$, there is a common covariance, $D=\Delta / 2$), which is defined as the Euclidean distance between $\boldsymbol{\mu}_1$ and the best linear discriminant hyperplane that yields equal error rates for $\pi_1$ and $\pi_2$ (after they have been transformed into canonical form).

	Although our model's extension to unequal covariance matrices is modest, deriving the expression for Asymptotic Relative Efficiency (ARE) as conducted in Theorem~\ref{thm1} is too complex. Consequently, we use simulation calculations to investigate the relationship between simulated relative efficiency, the degree of population separation, and the proportion of missing class labels.

	The first step is to examine the correlation between population separation and missing class labels.
	For each simulation combination, we generated $B=1000$ samples (replications) with a sample size of $n=200$ and an equal probability of choosing each population, i.e., $\pi_1=\pi_2$. For each combination of $p\in\{2, 3, 4\}$, $\lambda\in\{2, 3, 4\}$, $\Delta\in\{0.5,1,1.5,2,\ldots,12.5\}$, and $\boldsymbol{\xi}\in \{(3,1)^T, (3,4)^T \}$, the observations and missing label indicators in each sample were generated following the equations (\ref{eq:2-1}) and (\ref{eq:3-4}), respectively. Based on this simulation setup, we observe that the degree of separation, $D=\Delta/(1+\sqrt{\lambda})$, varies roughly from 0.2 to 4. Figure~\ref{fig:4} indicates that for all combinations, the average proportion of missing class labels across 1000 samples diminishes as the degree of separation between the two populations augments from high overlap to negligible overlap.
	
	The second step is to examine how relative efficiency correlates with the proportion of missing class labels, we employed the same simulation setup, but only selected $\Delta\in\{1, 2, 3\}$. In each replication, the estimates 
	$\hat{\boldsymbol{\theta}}_{\mathrm{CC}}$, $\hat{\boldsymbol{\theta}}_{\mathrm{PC}}^{\mathrm{ig})}$ and $\hat{\boldsymbol{\theta}}_{\mathrm{PC}}^{(\mathrm{full})}$ were computed using the general log-likelihood function, (\ref{eq:3-2}), (\ref{eq:3-1}), and (\ref{eq:3-9}), respectively. Since we cannot analytically calculate the conditional error rate (\ref{eq:2-4}), we utilized a holdout sample of size $0.2\times n$ in each simulation trial to assess it.

\begin{figure}[h!]
	\includegraphics[width=\linewidth]{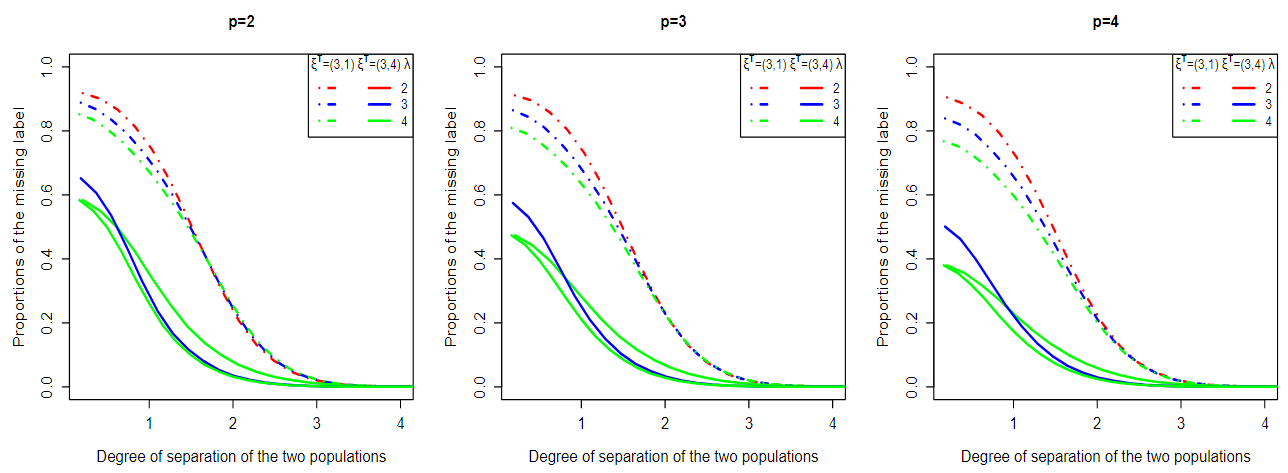}
	\caption{Plots of the degree of separation of two populations versus the average proportion of missing class labels over 1000 samples with sample size $n=200$ and $\pi_1=\pi_2$.}
	\label{fig:4}
\end{figure}

Then, the relative efficiency (RE) of $R(\hat{\boldsymbol{\theta}}_{\mathrm{PC}}^{(\mathrm{full})})$ compared to $R(\hat{\boldsymbol{\theta}}_{\mathrm{CC}})$ was estimated as
\begin{equation}\label{eq:4-5}
	\overline{\operatorname{RE}}\left\{R(\hat{\boldsymbol{\theta}}_{\mathrm{PC}}^{(\mathrm{full})}, \hat{\boldsymbol{\theta}}_{\mathrm{CC}})\right\}=\frac{B^{-1} \sum_{b=1}^B\left\{\operatorname{err}(\hat{\boldsymbol{\theta}}_{\mathrm{CC}}^{(b)})-\operatorname{err}(\boldsymbol{\theta})\right\}}{B^{-1} \sum_{b=1}^B\left\{\operatorname{err}(\hat{\boldsymbol{\theta}}_{\mathrm{PC}}^{(\mathrm{full}, b)})-\operatorname{err}(\boldsymbol{\theta})\right\}},
\end{equation}
and the $\mathrm{RE}$ of $R(\hat{\boldsymbol{\theta}}_{\mathrm{PC}}^{(\mathrm{full})})$ compared to $R(\hat{\boldsymbol{\theta}}_{\mathrm{PC}}^{(\mathrm{ig})})$ was estimated as
\begin{equation}\label{eq:4-6}
	\overline{\mathrm{RE}}\left\{R(\hat{\boldsymbol{\theta}}_{\mathrm{PC}}^{(\mathrm{full})}, \hat{\boldsymbol{\theta}}_{\mathrm{PC}}^{(\mathrm{ig})})\right\}=\frac{B^{-1} \sum_{b=1}^B\left\{\operatorname{err}(\hat{\boldsymbol{\theta}}_{\mathrm{PC}}^{\mathrm{(ig, b)}})-\operatorname{err}(\boldsymbol{\theta})\right\}}{B^{-1} \sum_{b=1}^B\left\{\operatorname{err}(\hat{\boldsymbol{\theta}}_{\mathrm{PC}}^{(\mathrm{full}, b)})-\operatorname{err}(\boldsymbol{\theta})\right\}} .
\end{equation}

\begin{table}[h!]
	\caption{Simulated relative efficiency $\overline{\mathrm{RE}}\left\{R(\hat{\boldsymbol{\theta}}_{\mathrm{PC}}^{\text {(full) }}, \hat{\boldsymbol{\theta}}_{\mathrm{CC}})\right\}$ with $g=2, \pi_1=\pi_2$, $\boldsymbol{\mu}_1=\mathbf{0}, \boldsymbol{\mu}_2=(\Delta, 0, \ldots, 0)^T, \boldsymbol{\Sigma}_1=\mathbf{I}$ and $\boldsymbol{\Sigma}_2=\lambda \mathbf{I}$ for $n=200$ and the missing labels $M_j$ generated with a missing label probability of $q(\mathbf{y}_j^{(b)} ; \boldsymbol{\theta}, \boldsymbol{\xi})$ in the $b$th trial (the proportion of missing class labels in parentheses)}
	\label{tab:1} 
	\begin{tabular}{llllllll}
	\toprule \multirow{2}{*}{$\lambda$} &\multirow{2}{*}{$\Delta$}& \multicolumn{3}{c}{$\boldsymbol{\xi}=(3,1)^T$} & \multicolumn{3}{c}{$\boldsymbol{\xi}=(3,4)^T$} \\
		\cmidrule{ 3 - 8 } 
		&  & $p=2$ & $p=3$ & $p=4$ & $p=2$ & $p=3$ & $p=4$ \\
	\midrule \multirow{3}{*}{2}& 1 & $0.21(0.90)$ & $0.27(0.89)$ & $0.28(0.89)$ & $1.20(0.66)$ & $1.71(0.60)$ & $1.71(0.60)$ \\
		& 2 & $0.26(0.81)$ & $0.70(0.79)$ & $0.50(0.79)$ & $1.93(0.41)$ & $1.95(0.37)$ & $4.00(0.37)$ \\
		& 3 & $1.07(0.65)$ & $1.09(0.62)$ & $1.09(0.62)$ & $1.71(0.20)$ & $1.58(0.18)$ & $1.58(0.18)$ \\
	\midrule & 1 & $0.28(0.87)$ & $0.53(0.84)$ & $0.57(0.82)$ & $1.52(0.61)$ & $1.54(0.46)$ & $1.54(0.46)$ \\
		3 & 2 & $0.34(0.79)$ & $0.73(0.74)$ & $0.87(0.74)$ & $2.93(0.44)$ & $1.71(0.32)$ & $1.76(0.32)$ \\
		& 3 & $0.35(0.67)$ & $1.34(0.62)$ & $1.34(0.62)$ & $3.10(0.24)$ & $1.39(0.19)$ & $1.39(0.19)$ \\
		\midrule \multirow{3}{*}{4} & 1 & $0.37(0.84)$ & $0.61(0.79)$ & $0.75(0.75)$ & $1.25(0.55)$ & $1.82(0.35)$ & $1.79(0.36)$ \\
		& 2 & $0.50(0.77)$ & $0.94(0.69)$ & $1.08(0.69)$ & $1.50(0.42)$ & $2.46(0.27)$ & $1.59(0.27)$ \\
		& 3 & $0.46(0.67)$ & $1.09(0.60)$ & $1.09(0.60)$ & $3.04(0.26)$ & $1.58(0.17)$ & $2.03(0.17)$ \\
	\bottomrule
	\end{tabular}
\end{table}

The nonparametric bootstrap with 1000 re-samples was used to assess the
variability of the estimates. All bootstrap standard errors are mostly small
relative to their estimand.
\begin{table}[h!]
	\caption{Simulated relative efficiency $\overline{\operatorname{RE}}\left\{R(\hat{\boldsymbol{\theta}}_{\mathrm{PC}}^{(\mathrm{full})}, \hat{\boldsymbol{\theta}}_{\mathrm{PC}}^{(\mathrm{ig})})\right\}$ with $g=2, \pi_1=\pi_2$, $\boldsymbol{\mu}_1=\mathbf{0}, \boldsymbol{\mu}_2=(\Delta, 0, \ldots, 0)^T, \boldsymbol{\Sigma}_1=\mathbf{I}$ and $\boldsymbol{\Sigma}_2=\lambda \mathbf{I}$ for $n=200$ and the missing labels $M_j$ generated with a missing label probability of $q(\mathbf{y}_j^{(b)} ; \boldsymbol{\theta}, \boldsymbol{\xi})$ in the $b$th trial (the proportion of missing class labels in parentheses)}\label{tab:2} 
	\begin{tabular}{llllllll}
		\toprule \multirow{2}{*}{$\lambda$} &\multirow{2}{*}{$\Delta$}& \multicolumn{3}{c}{$\boldsymbol{\xi}=(3,1)^T$} & \multicolumn{3}{c}{$\boldsymbol{\xi}=(3,4)^T$} \\
		\cmidrule{ 3 - 8 } 
		&  & $p=2$ & $p=3$ & $p=4$ & $p=2$ & $p=3$ & $p=4$ \\
		\midrule  & 1 & $1.28(0.90)$ & $1.15(0.89)$ & $1.28(0.89)$ & $3.60(0.66)$ & $4.15(0.60)$ & $4.15(0.60)$ \\
			2 & 2 & $1.96(0.81)$ & $2.22(0.79)$ & $2.29(0.79)$ & $4.15(0.41)$ & $3.11(0.37)$ & $6.83(0.37)$ \\
			& 3 & $3.03(0.65)$ & $2.95(0.62)$ & $2.95(0.62)$ & $1.92(0.20)$ & $2.41(0.18)$ & $2.41(0.18)$ \\
				\midrule \multirow{3}{*}{3} & 1 & $1.74(0.87)$ & $1.75(0.84)$ & $1.81(0.82)$ & $3.49(0.61)$ & $2.80(0.46)$ & $2.80(0.46)$ \\
			& 2 & $1.49(0.79)$ & $2.20(0.74)$ & $2.29(0.74)$ & $4.86(0.44)$ & $2.91(0.32)$ & $2.54(0.32)$ \\
			& 3 & $1.96(0.67)$ & $2.93(0.62)$ & $2.93(0.62)$ & $3.76(0.24)$ & $1.88(0.19)$ & $1.88(0.19)$ \\
			\midrule \multirow{3}{*}{4} & 1 & $1.71(0.84)$ & $1.71(0.79)$ & $2.10(0.75)$ & $2.98(0.55)$ & $2.55(0.35)$ & $3.01(0.36)$ \\
			& 2 & $2.01(0.77)$ & $2.46(0.69)$ & $2.62(0.69)$ & $2.80(0.42)$ & $3.87(0.27)$ & $2.52(0.27)$ \\
			& 3 & $2.96(0.67)$ & $2.51(0.60)$ & $2.76(0.60)$ & $1.21(0.26)$ & $1.97(0.17)$ & $1.81(0.17)$ \\
		\bottomrule
		\end{tabular}
\end{table}

In an analysis employing a multi-way ANOVA on Tables~\ref{tab:1} and \ref{tab:2}, we assessed the impact of variables $\lambda$, $\Delta$, $p$, and $\boldsymbol{\xi}$ on the proportion of missing class labels. The results indicate that only the factor $\Delta$ exerted a statistically significant influence on the outcome variable ($p$-value $< 0.001$). The variable $\boldsymbol{\xi}$ approached statistical significance with a p-value of 0.073, corroborating the findings presented in Figure~\ref{fig:4}.

Table~\ref{tab:1} reports the results for the $\mathrm{RE}$ of $R(\hat{\boldsymbol{\theta}}_{\mathrm{PC}}^{(\mathrm{full})})$ compared to $R(\hat{\boldsymbol{\theta}}_{\mathrm{CC}})$. 
For $\boldsymbol{\xi}=(3,1)^T$, we can see that when $\lambda$ is fixed and $\Delta$ increases (or the degree of separation of the two populations increases), the proportion of missing class labels drops (from approximately 0.9 to approximately 0.6), consistent with Figure~\ref{fig:4}, and all simulated $\mathrm{RE}$ values increase accordingly, except the value for $p=2, \lambda=4$ and $\Delta=3$. In other words, when excessive overlap decreases to moderate overlap with a relatively high proportion of missing class labels, the error rate of $R(\hat{\boldsymbol{\theta}}_{\mathrm{PC}}^{\text {(full) }})$ decreases to close to that of $R(\hat{\boldsymbol{\theta}}_{\mathrm{CC}})$, and in some cases, the error rate of $R(\hat{\boldsymbol{\theta}}_{\mathrm{PC}}^{(\mathrm{full})})$ is even smaller than that of $R(\hat{\boldsymbol{\theta}}_{\mathrm{CC}})$. A possible explanation for why the value for $p=2, \lambda=4$ and $\Delta=3$ does not grow is that there is still excessive overlap $(D=1)$ with a relatively high proportion of missing class labels and a dimension equal 2. For $\boldsymbol{\xi}=(3,4)^T$, the proportion of missing class labels decreases (from approximately 0.66 to approximately 0.17) when $\lambda$ is fixed and $\Delta$ increases (or the degree of separation of the two populations increases), which is similarly consistent with Figure~\ref{fig:4}. All simulated RE values are greater than 1, meaning that all classifiers $R(\hat{\boldsymbol{\theta}}_{\mathrm{PC}}^{(\mathrm{full})})$ perform better than $R(\hat{\boldsymbol{\theta}}_{\mathrm{CC}})$ when the size of the overlap region changes from large to moderate with a relatively low proportion of missing class labels. Another finding is that when $\lambda$ is fixed and $\Delta$ increases, some of simulated RE values initially increase and then decrease, which may be interpreted as cases in which, with a relatively small proportion of missing labels, the overlap first reaches the most appropriate overlap and then further decreases. This phenomenon is consistent with the results in Figure~\ref{fig:1}. Therefore, within the framework of the missing-data mechanism, for a two-class normal mixture model with unequal covariance matrices, the classifier $R(\hat{\boldsymbol{\theta}}_{\mathrm{PC}}^{(\mathrm{full})})$ performs better than $R(\hat{\boldsymbol{\theta}}_{\mathrm{CC}})$ when the overlap region ranges from moderate to small with the proportion of missing class labels from relatively high to low, or when the overlap region is large but the proportion of missing class labels is relatively low.

Table~\ref{tab:2} reports the results for the RE of $R(\hat{\boldsymbol{\theta}}_{\mathrm{PC}}^{(\mathrm{full})})$ compared to $R(\hat{\boldsymbol{\theta}}_{\mathrm{PC}}^{(\mathrm{ig})})$. 
We find that all RE values are greater than 1, indicating that the classifier $R(\hat{\boldsymbol{\theta}}_{\mathrm{PC}}^{\text {(full)}})$ performs better than $R(\hat{\boldsymbol{\theta}}_{\mathrm{PC}}^{\text {(ig)}})$. For $\boldsymbol{\xi}=(3,1)^T$ (a relatively high missing label proportion between 0.6 and 0.9 ), when $\lambda$ is fixed, the higher the $D$ value is, the higher the simulated RE value, implying that the classifier $R(\hat{\boldsymbol{\theta}}_{\mathrm{PC}}^{(\text {full})})$ performs much better than $R(\hat{\boldsymbol{\theta}}_{\mathrm{PC}}^{(\mathrm{ig})})$ as the overlap region changes from large to moderate. For $\boldsymbol{\xi}=(3,4)^T$ (a relatively low missing label proportion between 0.17 and 0.66), when $\lambda$ is fixed and $\Delta$ increases, some simulated RE values increase then decrease while others only decrease, which may be interpreted to indicate that under a relatively low proportion of missing class labels, the overlap must first reach the most appropriate value in some cases before decreasing further, while in other cases, the overlap starts at the most appropriate value and then decreases. This phenomenon is consistent with the results in Figure~\ref{fig:2}. Therefore, for a two-class normal mixture model with unequal covariance matrices, the classifier $R(\hat{\boldsymbol{\theta}}_{\mathrm{PC}}^{\text {(full)}})$ performs better than $R(\hat{\boldsymbol{\theta}}_{\mathrm{PC}}^{\text {(ig)}})$ regardless of the extent of overlap, and the performance gap between $R(\hat{\boldsymbol{\theta}}_{\mathrm{PC}}^{\text {(full)}})$ and $R(\hat{\boldsymbol{\theta}}_{\mathrm{PC}}^{\text {(ig) }})$ reaches its maximum when the partially classified sample has a moderate overlap region.

\subsection{Three-class normal model with unequal covariance matrices}
In this simulation, we consider a three-class normal model with unequal covariance matrices, for which is not straightforward to obtain the Euclidean
distances between the means or the best linear discriminant hyperplanes (the effort to do so is outside the scope of this study). Thus, we simply explore
the relationship between the simulated efficiency value and the proportion of missing class labels. We constructed $\boldsymbol{\pi}=(\pi_1, \pi_2, \pi_3)^T$, $\boldsymbol{\mu}_1=(-1,0, \ldots, 0)^T$, $\boldsymbol{\mu}_2=\boldsymbol{\mu}_3=(1,0, \ldots, 0)^T$, and $\boldsymbol{\Sigma}_i=\lambda_i \mathbf{I}(i=1,2,3)$. For each combination of $\boldsymbol{\pi}$ and $\lambda_i $ $(\boldsymbol{\pi}\in \{(0.5,0.35,0.15)^T,(1 / 3,1 / 3,1 / 3)^T\}$; $\lambda_1\in\{0.5,1.5,2.5\}$; $\lambda_2\in \{1,2,3\}$; $\lambda_3\in\{2,4,6\})$, we generated $B=1000$ samples with sample size $n=200$ and $p=2$. Note that if $\lambda_1=\lambda_2=2$, observations would be actually generated from a two-component normal model. Therefore, we do not consider this simulation setting. For the partially classified sample $\mathbf{x}_{\text{PC}}$, the missing label indicators were generated using the missing label probability in (\ref{eq:3-4}), where we set $\boldsymbol{\xi}=(3,7)^T$. In each replication, the estimates $\hat{\boldsymbol{\theta}}_{\mathrm{CC}}, \hat{\boldsymbol{\theta}}_{\mathrm{PC}}^{\text {(ig) }}$ and $\hat{\boldsymbol{\theta}}_{\mathrm{PC}}^{\text {(full)}}$ were computed using the general log-likelihood function, (\ref{eq:3-2}), (\ref{eq:3-1}) and (\ref{eq:3-9}), respectively, and we generated a holdout sample of size $0.2 \times n$, to assess the error rate via (\ref{eq:2-4}).

\begin{table}[h!]
	\caption{Simulated relative efficiency $\overline{\operatorname{RE}}\{R(\hat{\boldsymbol{\theta}}_{\mathrm{PC}}^{(\text {full)}})\}$ with $g=3, \boldsymbol{\mu}_1=$ $(-1,0, \ldots, 0)^T, \boldsymbol{\mu}_2=\boldsymbol{\mu}_3=(1,0, \ldots, 0)^T$, and $\boldsymbol{\Sigma}_i=\lambda_i \mathbf{I}(i=1,2,3)$ for $n=200$, and $p=2$ and the missing labels $M_j$ generated with a missing label probability of $q(\mathbf{y}_j^{(b)} ; \boldsymbol{\theta}, \boldsymbol{\xi})$ in the $b$-th trial with $\boldsymbol{\xi}=(3,7)^T$ the proportion of missing class labels in parentheses)}\label{tab:3} 
	\begin{tabular}{llllllll}
		\toprule \multirow{2}{*}{$\lambda_2$} & \multirow{2}{*}{$\lambda_3$} & \multicolumn{3}{c}{$\pi=(0.5,0.35,0.15)^T$} & \multicolumn{3}{c}{$\pi_1=\pi_2=\pi_3$} \\
		\cmidrule { 3 - 8 } & & $\lambda_1=0.5$ & $\lambda_1=1.5$ & $\lambda_1=2.5$ & $\lambda_1=0.5$ & $\lambda_1=1.5$ & $\lambda_1=2.5$ \\
		\midrule \multirow{3}{*}{1} & 2 & $1.70(0.38)$ & $0.49(0.61)$ & $0.34(0.66)$ & $1.05(0.58)$ & $0.51(0.75)$ & $0.32(0.79)$ \\
		& 4 & $2.40(0.30)$ & $1.83(0.54)$ & $0.90(0.61)$ & $2.54(0.47)$ & $0.75(0.69)$ & $1.07(0.74)$ \\
		& 6 & $2.93(0.25)$ & $2.27(0.48)$ & $1.41(0.56)$ & $2.30(0.39)$ & $1.81(0.63)$ & $1.13(0.70)$ \\
		\midrule  \multirow{3}{*}{2} & 2 & $0.58(0.43)$ & $0.27(0.66)$ & $0.20(0.73)$ & $1.22(0.67)$ & $0.39(0.83)$ & $0.21(0.85)$ \\
		& 4 & $1.81(0.38)$ & $0.91(0.65)$ & $0.31(0.73)$ & $1.05(0.60)$ & $0.44(0.81)$ & $0.31(0.85)$ \\
		& 6 & $1.76(0.32)$ & $1.30(0.60)$ & $0.70(0.70)$ & $2.15(0.53)$ & $0.61(0.77)$ & $0.48(0.82)$ \\
			\midrule  \multirow{3}{*}{3} & 2 & $0.67(0.40)$ & $0.34(0.35)$ & $0.22(0.72)$ & $0.82(0.65)$ & $0.53(0.83)$ & $0.37(0.86)$ \\
		& 4 & $1.05(0.40)$ & $0.55(0.67)$ & $0.28(0.77)$ & $0.81(0.64)$ & $0.50(0.86)$ & $0.28(0.90)$ \\
		& 6 & $1.22(0.36)$ & $1.07(0.65)$ & $0.51(0.76)$ & $1.12(0.58)$ & $0.55(0.81)$ & $0.31(0.87)$ \\
	\bottomrule
	\end{tabular}
\end{table}

The results for the $\mathrm{RE}$ of $R(\hat{\boldsymbol{\theta}}_{\mathrm{PC}}^{(\mathrm{full})})$ compared to $R(\hat{\boldsymbol{\theta}}_{\mathrm{CC}})$ and $R(\hat{\boldsymbol{\theta}}_{\mathrm{PC}}^{(\mathrm{ig})})$ are summarized in Table~\ref{tab:3} and Table~\ref{tab:4}, respectively.

Table~\ref{tab:3} reports that when the proportion of missing class labels is relatively high (over 0.65 in the case of $\boldsymbol{\pi}=(0.5,0.35,0.15)^T$ and over 0.74 in the case of $\pi_1=\pi_2=\pi_3)$, the values for the RE of $R(\hat{\boldsymbol{\theta}}_{\mathrm{PC}}^{(\mathrm{full})})$ compared to $R(\hat{\boldsymbol{\theta}}_{\mathrm{CC}})$ are less than 1, namely, the classifier $R(\hat{\boldsymbol{\theta}}_{\mathrm{CC}})$ performs better than $R(\hat{\boldsymbol{\theta}}_{\mathrm{PC}}^{(\mathrm{full})})$; when the proportion of missing class labels is relatively low (below 0.6 in the cases of both $\boldsymbol{\pi}=(0.5,0.35,0.15)^T$ and $\pi_1=\pi_2=\pi_3)$, most values for the $\mathrm{RE}$ of $R(\hat{\boldsymbol{\theta}}_{\mathrm{PC}}^{(\mathrm{full})})$ compared to $R(\hat{\boldsymbol{\theta}}_{\mathrm{CC}})$ are greater than 1, indicating that the classifier $R(\hat{\boldsymbol{\theta}}_{\mathrm{PC}}^{(\mathrm{full})})$ performs better than $R(\hat{\boldsymbol{\theta}}_{\mathrm{CC}})$.

Table~\ref{tab:4} reports the results for the RE of $R(\hat{\boldsymbol{\theta}}_{\mathrm{PC}}^{(\mathrm{full})})$ compared to $R(\hat{\boldsymbol{\theta}}_{\mathrm{PC}}^{(\mathrm{ig})})$ . We find that the proportions of the missing label are between 0.25 and 0.9, and all values of the RE of $R(\hat{\boldsymbol{\theta}}_{\mathrm{PC}}^{(\mathrm{full})})$ compared to $R(\hat{\boldsymbol{\theta}}_{\mathrm{PC}}^{\mathrm{ig})})$ are smaller than 1, which implies that the classifier $R(\hat{\boldsymbol{\theta}}_{\mathrm{PC}}^{(\mathrm{full})})$ outperforms $R(\hat{\boldsymbol{\theta}}_{\mathrm{PC}}^{(\mathrm{ig})})$ on a partially classified sample.

Although we have simply explored the relationship between the simulated efficiency value and the proportion of missing class labels in a three-class normal model with unequal covariance matrices, the results are consistent with the conclusions from Figure~\ref{fig:1} and \ref{fig:2}, and Section~\ref{SEC:5-1}.
\begin{table}[!h]
	\caption{Simulated relative efficiency $\overline{\operatorname{RE}}\{R(\hat{\boldsymbol{\theta}}_{\mathrm{PC}}^{(\mathrm{ig})})\}$ with $g=3, \boldsymbol{\mu}_1=$ $(-1,0, \ldots, 0)^T, \boldsymbol{\mu}_2=\boldsymbol{\mu}_3=(1,0, \ldots, 0)^T$, and $\boldsymbol{\Sigma}_i=\lambda_i \mathbf{I}(i=1,2,3)$. for $n=200$, and $p=2$ and the missing label $M_j$ generated with a missing label probability of $q(\mathbf{y}_j^{(b)} ; \boldsymbol{\theta}, \boldsymbol{\xi})$ in the $b$-th trial with $\boldsymbol{\xi}=(3,7)^T$ (the proportion of missing class labels in parentheses)}\label{tab:4} 
	\begin{tabular}{llllllll}
		\toprule \multirow{2}{*}{$\lambda_2$} & \multirow{2}{*}{$\lambda_3$} & \multicolumn{3}{c}{$\pi=(0.5,0.35,0.15)^T$} & \multicolumn{3}{c}{$\pi_1=\pi_2=\pi_3$} \\
		\cmidrule { 3 - 8 } & & $\lambda_1=0.5$ & $\lambda_1=1.5$ & $\lambda_1=2.5$ & $\lambda_1=0.5$ & $\lambda_1=1.5$ & $\lambda_1=2.5$ \\
\midrule \multirow{3}{*}{1} & 2 & $3.75(0.38)$ & $2.64(0.61)$ & $2.14(0.66)$ & $2.36(0.58)$ & $1.91(0.75)$ & $1.35(0.79)$ \\
	& 4 & $4.39(0.30)$ & $5.99(0.54)$ & $3.07(0.61)$ & $6.71(0.47)$ & $3.06(0.69)$ & $4.00(0.74)$ \\
	& 6 & $4.17(0.25)$ & $5.54(0.48)$ & $4.08(0.56)$ & $4.36(0.39)$ & $4.89(0.63)$ & $3.79(0.70)$ \\
	\midrule \multirow{3}{*}{2} & 2 & $1.49(0.43)$ & $1.64(0.66)$ & $1.34(0.73)$ & $3.04(0.67)$ & $1.97(0.83)$ & $1.57(0.85)$ \\
	& 4 & $3.52(0.38)$ & $3.22(0.65)$ & $1.82(0.73)$ & $2.42(0.6)$ & $1.96(0.81)$ & $1.59(0.85)$ \\
	& 6 & $2.84(0.32)$ & $4.11(0.60)$ & $2.50(0.70)$ & $4.68(0.53)$ & $2.80(0.77)$ & $2.19(0.82)$ \\
	\midrule \multirow{3}{*}{3} & 2 & $1.47(0.40)$ & $1.88(0.35)$ & $1.28(0.72)$ & $1.39(0.65)$ & $1.49(0.83)$ & $1.24(0.86)$ \\
	& 4 & $2.64(0.40)$ & $2.28(0.67)$ & $1.61(0.77)$ & $1.19(0.64)$ & $1.71(0.86)$ & $1.25(0.90)$ \\
	& 6 & $2.31(0.36)$ & $3.36(0.65)$ & $2.38(0.76)$ & $2.14(0.58)$ & $2.38(0.81)$ & $1.50(0.87)$ \\
	\bottomrule
\end{tabular}
\end{table}

\section{Application\label{SEC:6}}

A multitude of practical applications, ranging from medical diagnostics to neuroscience investigations, primarily rely on human expertise for assigning class labels. This methodology of manual dataset annotation gives rise to a systematic mechanism for missingness. Intriguingly, when such missingness is accurately deciphered and leveraged, it can yield significant insights. Our research provides compelling evidence based on two authentic datasets: specifically, the Interneuron dataset and the Skin Lesion dataset. Recent literature, including works by \cite{Koosowska2023microRNAdependentRO,DiVolo2021,Lucius2020,2023IJMS}, highlights the growing interest in these datasets due to their potential to advance understanding in neuroscience and dermatology, respectively. Our investigations reveal a significant correlation between the mechanism for missingness, responsible for the absence of some labels, and the concept of Shannon entropy. This novel finding has multidisciplinary implications across numerous fields, including but not limited to statistics, computer science, and healthcare. Furthermore, we demonstrated that a classifier accounting for this missingness mechanism exhibits superior performance compared to a classifier that disregards it. Intriguingly, it even surpasses the performance of a classifier built on a completely classified sample. These findings have considerable potential to benefit researchers and practitioners across diverse fields by improving classification accuracy and enhancing data utilization. By effectively addressing the issue of missingness in datasets, we catalyze the exploration of new horizons for data-driven insights and research potential. An example with simulated data, intended to demonstrate the application of our methodology, is available in the supplementary material.

\subsection{Interneuron dataset}
There is currently no unique catalog of cortical GABAergic interneuron types.
\cite{mihaljevic2019classification} asked 48 leading neuroscientists to classify 320 interneurons by inspecting images of their morphology. Each neuroscientist used
a web application to classify the interneurons, and each image was resized to
$32 \times 32$ pixels. A total of 48 neuroscientists participated in the study. Here,
we consider only the 42 neuroscientists who fully classified all 320 neurons.
We consider 2 classes in this dataset: transcolumnar and intracolumnar. We
discard those observations for which more than or equal to half of the 42 neuroscientists viewed an entry as not applicable to a given interneuron . Thus,
we consider a total of $n = 304$ observations  {with a dimension of 1024}. The ground truth for each image
was generated by agreement among at least half of the 42 neuroscientists.

We select five neuroscientists who reviewed the images and 3D visualizations to determine whether an interneuron was transcolumnar or intracolumnar. We form a consensus labeling using the individual neuroscientists’ labels.
Observations for which all five neuroscientists agreed are treated as labeled
data, corresponding to $n_l = 173$. Observations for which fewer than five neuroscientists agreed are treated as unlabeled data, corresponding to $n_u = 131$.
The dimensionality of the dataset is reduced to 3 features via principal component analysis (PCA) prior to classification.  The decision to select $p=3$ principal components was guided by the analysis of a scree plot, a graphical representation of the eigenvalues derived from the correlation or covariance matrix. The 'elbow' in the plot, where the subsequent eigenvalues diminish and become approximately equal, was identified at $p=3$, hence substantiating our choice.

We fit a skew t-mixture model \cite{lee2014finite} to estimate the
classification entropy of each observation. Figure~\ref{fig:5} (a) compares the empirical
cumulative distribution functions of the estimated entropy distributions in the
labeled and unlabeled groups. Figure~\ref{fig:5} (b) presents a Nadaraya–Watson kernel
estimate of the probability of missing class labels. From Figure~\ref{fig:5} (a), we find
that the unlabeled observations typically have higher entropy than the labeled
observations. Figure~\ref{fig:5} (b) shows that the estimated probability of missing class
labels decreases as the negative log entropy increases. This is in accordance
with the relation of equation (\ref{eq:3-4}). The higher the entropy is, the higher the
probability of missing class labels. Therefore, we fit a two-class normal model with unequal covariance matrices to compare the three classifiers $R(\hat{\boldsymbol{\theta}}_{\mathrm{CC}})$, $R(\hat{\boldsymbol{\theta}}_{\mathrm{PC}}^{\mathrm{ig})})$ and $R(\hat{\boldsymbol{\theta}}_{\mathrm{PC}}^{(\mathrm{full})})$.

We apply fivefold cross-validation (CV) to test the error rates of the three
classifiers, splitting the interneuron dataset into a test set and a training set.
\begin{figure}[h!]
	\includegraphics[width=\linewidth]{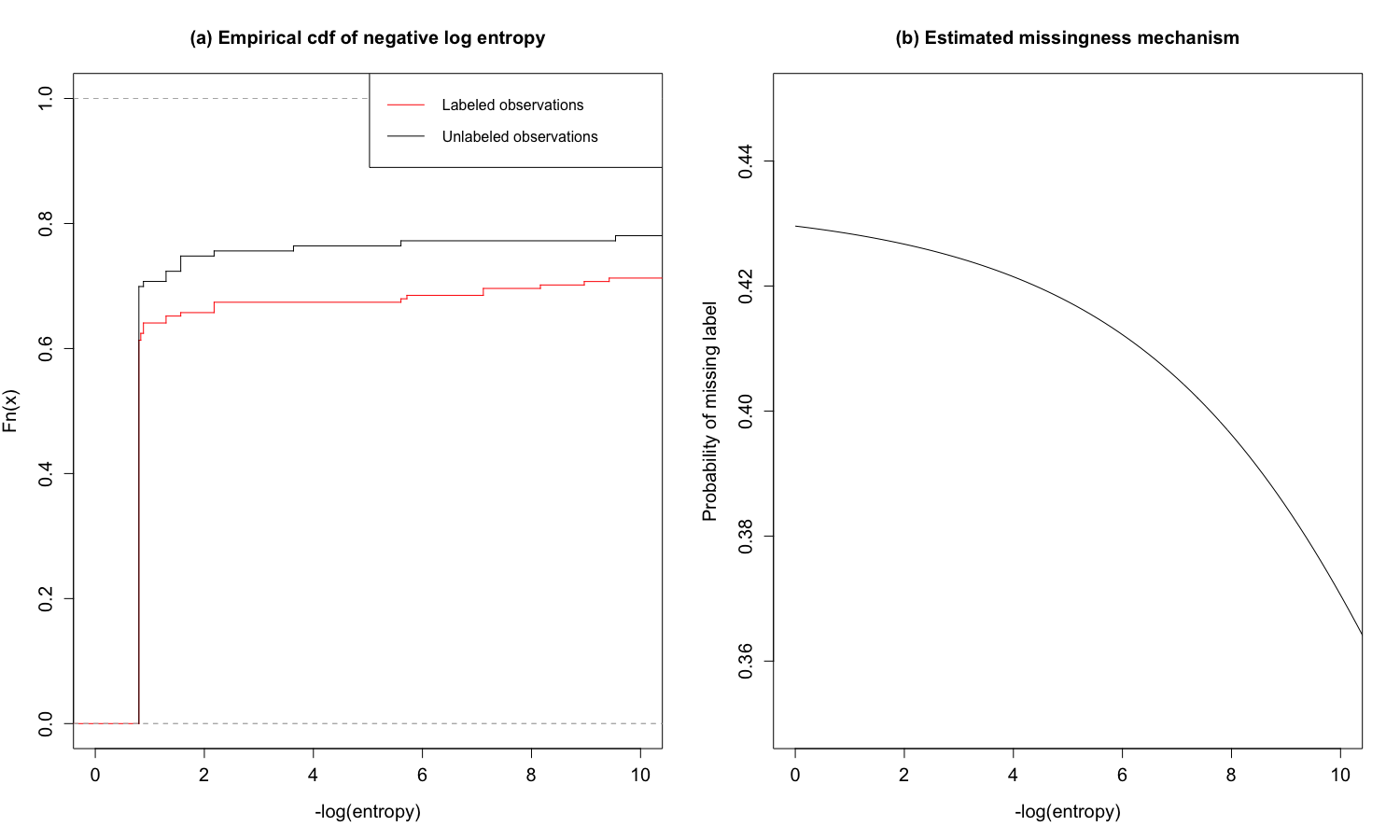}
	\caption{Analysis of the interneuron dataset with regard to the relationship
		between the entropy and the labeled and unlabeled observations.}
	\label{fig:5}
\end{figure}
\begin{table}[h!]
	\caption{Summary statistics for the interneuron dataset. There are $n=304$  observations with dimensions of $p=3$. There are $g=2$ classes (transcolumnar and intracolumnar).}\label{tab:5} 
	\begin{tabular}{ccccccc}
		\toprule
		 Classifier & $n$ & $p$ & $g$ & $n_c$ (classified) & $n_u$ (unclassified) & 5-fold CV error rate \\
			\midrule$R(\hat{\boldsymbol{\theta}}_{\mathrm{CC}})$ & & & & 304 & 0 & 0.2466 \\
			$R(\hat{\boldsymbol{\theta}}_{\mathrm{PC}}^{\text {(ig) }})$ & 304 &3& 2 & 173 & 131 & 0.3406 \\
			$R(\hat{\boldsymbol{\theta}}_{\mathrm{PC}}^{(\text {full })})$ & & & & 173 & 131 & 0.2629 \\
		\bottomrule
	\end{tabular}
\end{table}

The dataset is randomly split into five subsets of equal sizes. Of the five subsets, a single subset is retained as the validation data for testing the model,
and the remaining four subsets are used as the training data. This CV process
is repeated five times, with each of the five subsets used exactly once as the
validation data. The five results are then averaged to produce a single estimate.
The error rate summary is given in Table~\ref{tab:5}, and we find that the fivefold CV
error rate from the partially classified sample based on the missingness mechanism is smaller than that obtained when ignoring the missingness mechanism
and close to the fivefold CV error rate from the completely classified sample.

\subsection{Skin lesion dataset}
Dermatoscopic images can be accessed and downloaded from Harvard Dataverse provided by \cite{DVN/DBW86T_2018}. We select 3 classes e.g., benign keratosis (BKL), dermatofibroma (DF), and vascular skin lesions (VASC). The
dataset includes a set of ground–truth labels of two types. Some labels have
been confirmed by histopathology (performed by specialized dermatopathologists), while the rest have been confirmed either by follow-up, expert consensus, or in-vivo confocal microscopy. We treat the observations diagnosed by
dermatopathologists, as labeled data and the rest of the cases as data with
missing labels. Thus, the subset we consider contains $n = 1356$ observations,
with $n_l = 888$ labeled observations and $n_u = 468$ unlabeled observations.
To reduce the dimensionality of the dataset, we employed sparse linear discriminant analysis (LDA) \cite{clemmensen2011sparse} to select a subset of four features for model-based clustering.
Employing LDA for dimensionality reduction in a skin lesion dataset was motivated by the availability of ground-truth labels. These labels served as response variables in LDA, facilitating the selection of four discriminative features, thereby reducing feature redundancy and enhancing subset quality.

Similarly, we fit a t-mixture model \cite{lee2014finite} to estimate
the classification entropy of each observation. Figure~\ref{fig:6} (a) compares box-plots
of the estimated entropies in the labeled and unlabeled groups. Figure~\ref{fig:6} (b)
presents a Nadaraya–Watson kernel estimate of the probability of missing class
labels. From Figure~\ref{fig:6} (a), we find that the unlabeled observations typically
have higher entropy than the labeled observations. Figure~\ref{fig:6} (b) shows the estimated probability of missing class labels decreases as the negative log entropy
increases. This is in accordance with the relation of equation (\ref{eq:3-4}). Therefore,
we fit a three-class normal model with unequal covariance matrices to compare
the three classifiers $R(\hat{\boldsymbol{\theta}}_{\mathrm{CC}}), R(\hat{\boldsymbol{\theta}}_{\mathrm{PC}}^{(\mathrm{ig})})$ and $R(\hat{\boldsymbol{\theta}}_{\mathrm{PC}}^{(\mathrm{full})})$.
\begin{table}[h!]
	\caption{Summary statistics for the skin lesion dataset. There are $n = 1356$
		 observations with $p = 4$ dimensions. There are $g = 3$ classes (BKL,
		DF, and VASC).}\label{tab:6} 
	\begin{tabular}{ccccccc}
		\toprule
		Classifier & $n$ & $p$ & $g$ & $n_c$ (classified) & $n_u$ (unclassified) & 5-fold CV error rate \\
		\midrule$R(\hat{\boldsymbol{\theta}}_{\mathrm{CC}})$ & & & & 1356 &0 &0.1998 \\
		$R(\hat{\boldsymbol{\theta}}_{\mathrm{PC}}^{\text {(ig) }})$ & 1356 & 4 & 3 & 888& 468& 0.2005 \\
		$R(\hat{\boldsymbol{\theta}}_{\mathrm{PC}}^{(\text {full })})$ & & & & 888 &468 &0.1969 \\
		\bottomrule
	\end{tabular}
\end{table}

We apply fivefold CV to test the error rates of three classifiers, splitting the
skin lesion dataset into a test set and a training set. The error rate summary
is given in Table~\ref{tab:6}. We find that the error rate from the partially classified
sample based on the missingness mechanism is smaller than that obtained
when the missingness mechanism is ignored and even better than that from
the completely classified sample. But the difference between the error rates
are very small, suggesting that the unlabeled observations must be easy to
classify in the iterative fitting of  the EM algorithm and entropies in Figure~\ref{fig:6}(a) are relatively low which can explain that $R(\hat{\boldsymbol{\theta}}_{\mathrm{PC}}^{(\mathrm{full})})$ does not do much better than $R(\hat{\boldsymbol{\theta}}_{\mathrm{PC}}^{(\mathrm{ig})})$ and is close to $R(\hat{\boldsymbol{\theta}}_{\mathrm{CC}})$.

Note that in using cross-validation to estimate the error rate of the various rules, the rule $R(\hat{\boldsymbol{\theta}}_{\mathrm{CC}})$ has an advantage over $R(\hat{\boldsymbol{\theta}}_{\mathrm{PC}})$ in their application to the unclassified feature vectors since $R(\hat{\boldsymbol{\theta}}_{\mathrm{CC}})$ is formed using the true values of the labels for the unclassified features. Bearing this in mind, the full rule based on the partially classified sample does well to be fairly close to that of the rule based on the completely classified sample in the former Interneuron dataset and even better in this Skin Lesion dataset.
\begin{figure}[h!]
	\includegraphics[width=\linewidth]{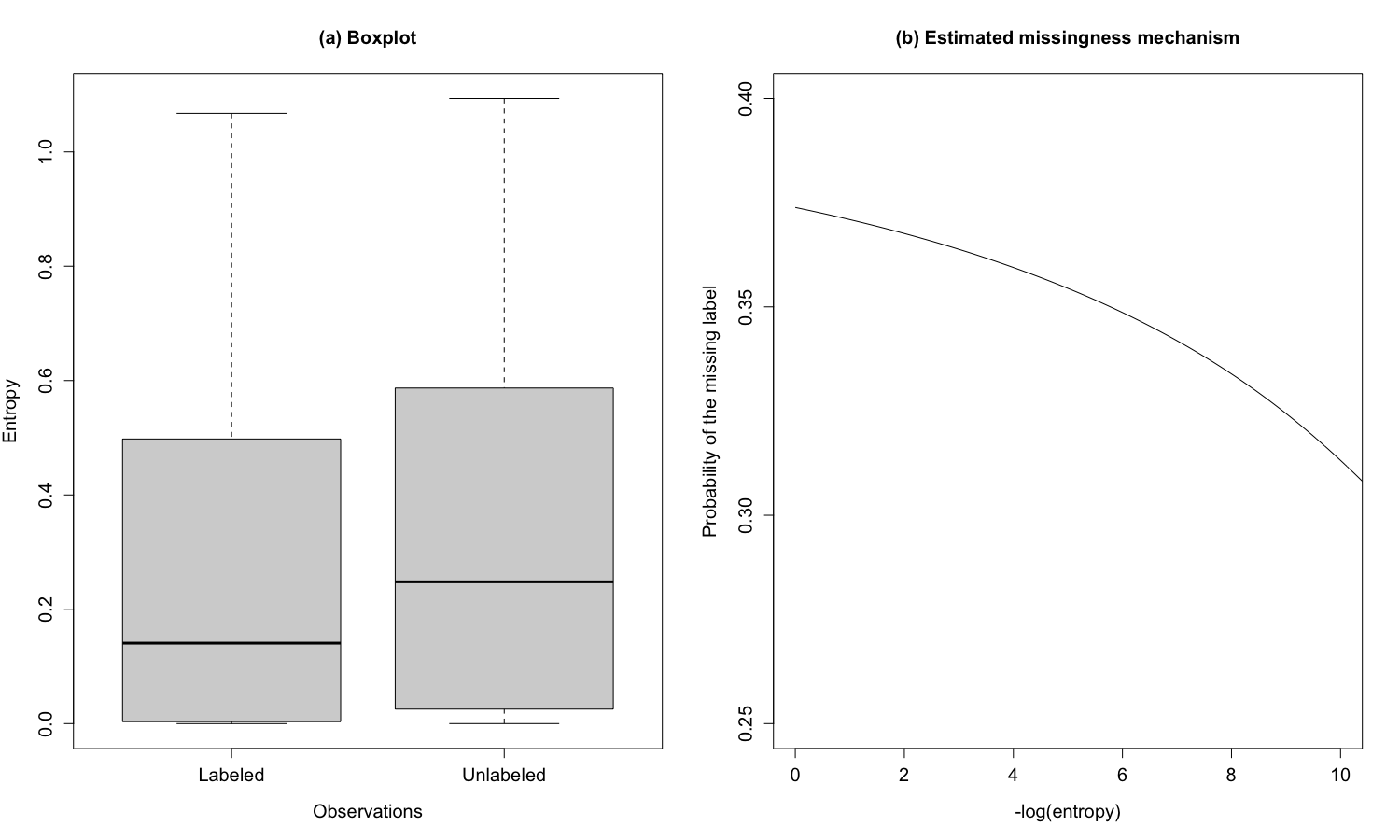}
	\caption{Analysis of the skin lesion dataset with regard to the relationship
		between the entropy and the labeled and unlabeled observations.}
	\label{fig:6}
\end{figure}

 \section{Discussion\label{SEC:7}}
In the present study, we have delved into a statistical semi-supervised learning
(SSL) methodology predicated on a finite normal mixture model, taking into
account the MAR concept of missing data. The missingness mechanism is articulated as a logistic function to model the conditional
probability of a label being missing, contingent on the log entropy of the observation. We have scrutinized the conditions under which the classifier $R_{\mathrm{PC}}^{(\text {full})}$ is more advantageous to use than $R_{\mathrm{CC}}$ for a two-class normal homoscedastic model. Our findings suggest that the classifier $R_{\mathrm{PC}}^{(\mathrm{full})}$ on a partially classified sample can surpass the completely supervised learning classifier $R_{\mathrm{CC}}$ when the size of the overlap region ranges from moderate to small and the proportion of missing class labels ranges from relatively high to low. Alternatively, this can also occur when the overlap region is large but the proportion of missing labels is relatively low. Furthermore, the classifier $R_{\mathrm{PC}}^{(\mathrm{full})}$ consistently outperforms $R_{\mathrm{PC}}^{(\mathrm{ig})}$ irrespective of the extent of overlap or the proportion of missing class labels. We have conducted simulation experiments to extend the analysis to two-class and three-class normal mixture models with unequal covariance matrices, and the results further corroborate the above analysis. We have also compared the three classifiers on two real datasets, and the results align with our expectations and the theoretical results.

While our study provides preliminary evidence in support of the efficacy of the proposed semi-supervised learning approach, it is important to acknowledge the limitation of the size of the real-world datasets utilized. The datasets were manually labeled by domain experts, a process that is both labor-intensive and costly, thus limiting the scale of data available for our study. Moreover, we designed our analyses to focus on a particular type of missingness mechanism.
To mitigate this limitation, we included a simulated example based on synthetic datasets generated to mimic the specific missingness mechanism in the supplementary.  The results from these additional experiments are consistent with those derived from the real-world datasets, thereby enhancing the robustness and credibility of our findings.
Future research would benefit from the inclusion of larger and more diverse datasets to validate the generalizability and robustness of the proposed semi-supervised learning method.

In light of the current research landscape, there has been a recent resurgence of interest in SSL in the machine learning community due to significant empirical progress on benchmark image and text classification datasets.
However, from a statistical perspective, the theoretical understanding of these
modern SSL algorithms is still limited \cite{ahfock2023semi}. One of
the most intuitive approaches to SSL is self-training, an iterative method for
learning with alternating steps between generating pseudo-labels for the unlabeled observations and then training a classifier using both the labeled and
pseudo-labeled data. Theoretical analysis of SSL has been limited, but recent
work by \cite{ahfock2020apparent} provided an asymptotic basis on how
to increase the accuracy of the commonly used linear discriminant function
formed from a partially classified sample in SSL. This increase in accuracy
can be of sufficient magnitude for this SSL-based classifier to have a smaller
error rate than that if it were formed from a completely classified sample.
Given the scarcity of theoretical analysis in this area, future work will encompass mathematical and empirical studies of more complex models. This will
include exploring the potential of the missingness mechanism to provide additional information and improve classification accuracy, as well as investigating
the conditions under which these gains can be achieved.


\bibliographystyle{spmpsci}      
\bibliography{mythesisbib}   

\begin{thebibliography}{10}
\providecommand{\url}[1]{{#1}}
\providecommand{\urlprefix}{URL }
\expandafter\ifx\csname urlstyle\endcsname\relax
  \providecommand{\doi}[1]{DOI~\discretionary{}{}{}#1}\else
  \providecommand{\doi}{DOI~\discretionary{}{}{}\begingroup
  \urlstyle{rm}\Url}\fi

\bibitem{ahfock2020apparent}
Ahfock, D., McLachlan, G.J.: An apparent paradox: a classifier based on a
  partially classified sample may have smaller expected error rate than that if
  the sample were completely classified.
\newblock Statistics and Computing pp. 1--12 (2020)

\bibitem{ahfock2023semi}
Ahfock, D., McLachlan, G.J.: Semi-supervised learning of classifiers from a
  statistical perspective: A brief review.
\newblock Econometrics and Statistics \textbf{26}, 124--138 (2023)

\bibitem{bartlett1963discrimination}
Bartlett, M., Please, N.: Discrimination in the case of zero mean differences.
\newblock Biometrika \textbf{50}(1/2), 17--21 (1963)

\bibitem{blum1998combining}
Blum, A., Mitchell, T.: Combining labeled and unlabeled data with co-training.
\newblock In: Proceedings of the eleventh annual conference on Computational
  learning theory, pp. 92--100 (1998)

\bibitem{chapelle2006semi}
Chapelle, O., Sch{\"o}lkopf, B., Zien, A.: Semi-supervised learning. adaptive
  computation and machine learning.
\newblock MIT Press, Cambridge, MA, USA. Cited in page (s) \textbf{21}(1), 2
  (2010)

\bibitem{chawla2005learning}
Chawla, N.V., Karakoulas, G.: Learning from labeled and unlabeled data: An
  empirical study across techniques and domains.
\newblock Journal of Artificial Intelligence Research \textbf{23}, 331--366
  (2005)

\bibitem{clemmensen2011sparse}
Clemmensen, L., Hastie, T., Witten, D., Ersb{\o}ll, B.: Sparse discriminant
  analysis.
\newblock Technometrics \textbf{53}(4), 406--413 (2011)

\bibitem{come2009learning}
C{\^o}me, E., Oukhellou, L., Denoeux, T., Aknin, P.: Learning from partially
  supervised data using mixture models and belief functions.
\newblock Pattern recognition \textbf{42}(3), 334--348 (2009)

\bibitem{efron1975efficiency}
Efron, B.: The efficiency of logistic regression compared to normal
  discriminant analysis.
\newblock Journal of the American Statistical Association \textbf{70}(352),
  892--898 (1975)

\bibitem{fujino2008semisupervised}
Fujino, A., Ueda, N., Saito, K.: Semisupervised learning for a hybrid
  generative/discriminative classifier based on the maximum entropy principle.
\newblock IEEE Transactions on Pattern Analysis and Machine Intelligence
  \textbf{30}(3), 424--437 (2008)

\bibitem{gilbert1969effect}
Gilbert, E.S.: The effect of unel variance-covariance matrices on fisher's
  linear discriminant function.
\newblock Biometrics pp. 505--515 (1969)

\bibitem{han1969distribution}
Han, C.P.: Distribution of discriminant function when covariance matrices are
  proportional.
\newblock The Annals of Mathematical Statistics \textbf{40}(3), 979--985 (1969)

\bibitem{hawkins1982extension}
Hawkins, D.M., Raath, E.L.: An extension of geisser's discrimination model to
  proportional covariance matrices.
\newblock The Canadian Journal of Statistics/La Revue Canadienne de Statistique
  pp. 261--270 (1982)

\bibitem{huang2010semi}
Huang, J.T., Hasegawa-Johnson, M.: Semi-supervised training of gaussian mixture
  models by conditional entropy minimization.
\newblock In: Eleventh Annual Conference of the International Speech
  Communication Association (2010)

\bibitem{joachims1999transductive}
Joachims, T., et~al.: Transductive inference for text classification using
  support vector machines.
\newblock In: Icml, vol.~99, pp. 200--209 (1999)

\bibitem{kim2007texture}
Kim, S.C., Kang, T.J.: Texture classification and segmentation using wavelet
  packet frame and gaussian mixture model.
\newblock Pattern recognition \textbf{40}(4), 1207--1221 (2007)

\bibitem{Koosowska2023microRNAdependentRO}
Kołosowska, K., Schratt, G., Winterer, J.: microrna-dependent regulation of
  gene expression in gabaergic interneurons.
\newblock Frontiers in Cellular Neuroscience \textbf{17} (2023)

\bibitem{lanckriet2004learning}
Lanckriet, G.R., Cristianini, N., Bartlett, P., Ghaoui, L.E., Jordan, M.I.:
  Learning the kernel matrix with semidefinite programming.
\newblock Journal of Machine learning research \textbf{5}(Jan), 27--72 (2004)

\bibitem{lee2014finite}
Lee, S., McLachlan, G.J.: Finite mixtures of multivariate skew t-distributions:
  some recent and new results.
\newblock Statistics and Computing \textbf{24}(2), 181--202 (2014)

\bibitem{Lucius2020}
Lucius, M., All, J.D., All, J.A.D., Belvisi, M., Radizza, L., Lanfranconi, M.,
  Lorenzatti, V., Galmarini, C.M.: Deep neural frameworks improve the accuracy
  of general practitioners in the classification of pigmented skin lesions.
\newblock Diagnostics \textbf{10} (2020)

\bibitem{lyu2023gmmsslm}
Lyu, Z., Ahfock, D., Thompson, R., McLachlan, G.J.: gmmsslm: Semi-supervised
  gaussian mixture modeling with a missing data mechanism in r.
\newblock arXiv preprint arXiv:2302.13206  (2023)

\bibitem{marks1974discriminant}
Marks, S., Dunn, O.J.: Discriminant functions when covariance matrices are
  unequal.
\newblock Journal of the American Statistical Association \textbf{69}(346),
  555--559 (1974)

\bibitem{mclachlan1975iterative}
McLachlan, G.J.: Iterative reclassification procedure for constructing an
  asymptotically optimal rule of allocation in discriminant analysis.
\newblock Journal of the American Statistical Association \textbf{70}(350),
  365--369 (1975)

\bibitem{mclachlan1975some}
McLachlan, G.J.: Some expected values for the error rates of the sample
  quadratic discriminant function1.
\newblock Australian Journal of Statistics \textbf{17}(3), 161--165 (1975)

\bibitem{mclachlan1977estimating}
McLachlan, G.J.: Estimating the linear discriminant function from initial
  samples containing a small number of unclassified observations.
\newblock Journal of the American statistical association \textbf{72}(358),
  403--406 (1977).
\newblock \doi{10.1080/01621459.1977.10481009}

\bibitem{mclachlan1989mixture}
McLachlan, G.J., Gordon, R.: Mixture models for partially unclassified data: a
  case study of renal venous renin in hypertension.
\newblock Statistics in Medicine \textbf{8}(10), 1291--1300 (1989).
\newblock \doi{10.1002/sim.4780081012}

\bibitem{mealli2015clarifying}
Mealli, F., Rubin, D.B.: Clarifying missing at random and related definitions,
  and implications when coupled with exchangeability.
\newblock Biometrika \textbf{102}(4), 995--1000 (2015)

\bibitem{mihaljevic2019classification}
Mihaljevi{\'c}, B., Benavides-Piccione, R., Bielza, C., Larra{\~n}aga, P.,
  DeFelipe, J.: Classification of gabaergic interneurons by leading
  neuroscientists.
\newblock Scientific data \textbf{6}(1), 1--6 (2019)

\bibitem{o1978normal}
O'Neill, T.J.: Normal discrimination with unclassified observations.
\newblock Journal of the American Statistical Association \textbf{73}(364),
  821--826 (1978)

\bibitem{pan2006semi}
Pan, W., Shen, X., Jiang, A., Hebbel, R.P.: Semi-supervised learning via
  penalized mixture model with application to microarray sample classification.
\newblock Bioinformatics \textbf{22}(19), 2388--2395 (2006)

\bibitem{rubin1976inference}
Rubin, D.B.: Inference and missing data.
\newblock Biometrika \textbf{63}(3), 581--592 (1976)

\bibitem{szczurek2010introducing}
Szczurek, E., Biecek, P., Tiuryn, J., Vingron, M.: Introducing knowledge into
  differential expression analysis.
\newblock Journal of Computational Biology \textbf{17}(8), 953--967 (2010)

\bibitem{szummer2001partially}
Szummer, M., Jaakkola, T.: Partially labeled classification with markov random
  walks.
\newblock Advances in neural information processing systems \textbf{14} (2001)

\bibitem{DVN/DBW86T_2018}
Tschandl, P.: {The HAM10000 dataset, a large collection of multi-source
  dermatoscopic images of common pigmented skin lesions} (2018)

\bibitem{vapnik1998support}
Vapnik, V.: The support vector method of function estimation pp. 55--85 (1998)

\bibitem{DiVolo2021}
Volo, M.D., Destexhe, A.: Optimal responsiveness and information flow in
  networks of heterogeneous neurons.
\newblock Scientific Reports \textbf{11}, 17611 (2021).
\newblock \doi{10.1038/s41598-021-96745-2}

\bibitem{2023IJMS}
{Wang}, Y., {Wang}, T.t., {Montero-Pedrazuela}, A., {Guada{\~n}o-Ferraz}, A.,
  {Rausell}, E.: {Thyroid Hormone Transporters MCT8 and OATP1C1 Are Expressed
  in Pyramidal Neurons and Interneurons in the Adult Motor Cortex of Human and
  Macaque Brain}.
\newblock International Journal of Molecular Sciences \textbf{24} (2023)

\bibitem{zhou2003learning}
Zhou, D., Bousquet, O., Lal, T., Weston, J., Sch{\"o}lkopf, B.: Learning with
  local and global consistency.
\newblock Advances in neural information processing systems \textbf{16} (2003)

\end{thebibliography}


\end{document}